\documentclass{article}
\usepackage[letterpaper,bindingoffset=0.2in,%
left=1in,right=1in,top=1in,bottom=1in,%
footskip=.25in]{geometry}

\usepackage[utf8]{inputenc} 
\usepackage[T1]{fontenc}    
\usepackage{hyperref}       
\usepackage{url}            
\usepackage{booktabs}       
\usepackage{amsfonts}       
\usepackage{nicefrac}       
\usepackage{microtype}      
\usepackage{amsmath}
\usepackage{amsthm}
\usepackage{xcolor}
\usepackage{multirow}
\usepackage{stmaryrd}
\usepackage[mathscr]{euscript}
\usepackage{upgreek}        
\usepackage{stackengine}
\usepackage{enumitem}
\usepackage{caption}
\usepackage{subcaption}

\newcommand{\RomanNumeralCaps}[1]{\MakeUppercase{\romannumeral #1}} 

\newcommand{\T}[2][]{#1\mathscr{\MakeUppercase{#2}}}
\newcommand{\Expect}{\operatorname{\mathbb{E}}}
\newcommand{\Prob}[1]{\mathbb{P}\left\{ #1 \right\}}
\newcommand{\RR}{\mathbb{R}}
\newcommand{\norm}[1]{\Vert #1 \Vert}
\newcommand{\fnorm}[1]{\norm{#1}_{\mathrm{F}}}
\newcommand{\snorm}[1]{\norm{#1}_\square}
\newcommand{\nucnorm}[1]{\norm{#1}_\star}
\newcommand{\maxnorm}[1]{\norm{#1}_{\infty}}

\newcommand{\obs}[1]{#1_{\mathrm{obs}}}

\newcommand{\rank}[1]{\mathrm{rank}({#1})}
\def\Ah{\smash{\widehat A}}

\def\Uh{\smash{\widehat U}}
\def\uh{\smash{\widehat u}}
\def\Vh{\smash{\widehat V}}
\def\Xh{\widehat X}
\def\Th{\widehat T}
\def\Ph{\widehat P}
\def\vh{\smash{\widehat v}}

\usepackage{graphicx}

\theoremstyle{plain}
\newtheorem{theorem}{Theorem}
\newtheorem{lemma}[theorem]{Lemma}

\definecolor{medblue}{rgb}{0,0,.75}
\usepackage[noend]{algpseudocode}
\usepackage{algorithm,algorithmicx}
\algrenewcommand\alglinenumber[1]{\sf\tiny\color{medblue}{#1}\quad}
\algrenewcommand\algorithmicrequire{\textbf{Input:}}
\algrenewcommand\algorithmicensure{\textbf{Output:}}
\algdef{SE}[DOWHILE]{Do}{doWhile}{\algorithmicdo}[1]{\algorithmicwhile\ #1}%

\def \figurepath {figures/}

\title{TenIPS: Inverse Propensity Sampling for Tensor Completion}
\date{}
\author{Chengrun Yang\footnote{School of Electrical and Computer Engineering, Cornell University, Ithaca, NY, USA; \texttt{cy438@cornell.edu}}, Lijun Ding\footnote{School of Operations Research and Information Engineering, Cornell University,
		Ithaca, NY, USA; \texttt{ld446@cornell.edu}}, Ziyang Wu\footnote{Department of Computer Science, Cornell University,
		Ithaca, NY, USA;
		\texttt{zw287@cornell.edu}}, Madeleine Udell\footnote{School of Operations Research and Information Engineering, Cornell University,
		Ithaca, NY, USA;
		\texttt{udell@cornell.edu}}}

\begin{document}

\maketitle

\begin{abstract}
Tensors are widely used to represent multiway arrays of data.
The recovery of missing entries in a tensor has been extensively studied,
generally under the assumption that entries are missing completely at random (MCAR).
However, in most practical settings, observations are missing not at random (MNAR):
the probability that a given entry is observed (also called the propensity)
may depend on other entries in the tensor
or even on the value of the missing entry.
In this paper, we study the problem of completing a partially observed tensor with MNAR observations,
without prior information about the propensities.
To complete the tensor,
we assume that both the original tensor and the tensor of propensities have low multilinear rank.
The algorithm first estimates the propensities using a convex relaxation
and then predicts missing values using a higher-order SVD approach, reweighting the observed tensor by the inverse propensities.
We provide finite-sample error bounds on the resulting complete tensor.
Numerical experiments demonstrate the effectiveness of our approach.
\end{abstract}

\section{Introduction}
Tensor completion is gaining increasing popularity and is one of the major tensor-related research topics.
The literature we survey here is by no means exhaustive.
A straightforward method is to flatten a tensor along one of its dimensions to a matrix and then pick one of the extensively studied matrix completion algorithms \cite{candes2010power, negahban2012restricted, cai2013max}.
However, this method neglects the multiway structure along all other dimensions and does not make full use of the combinatorial relationships.
Instead, it is common to assume that the tensor is low rank along each mode.
Tensors differ from matrices in having many incompatible notions of
rank and low rank decompositions,
including CANDECOMP/PARAFAC (CP) \cite{carroll1970analysis, harshman1970foundations},
Tucker \cite{tucker1966some}
and tensor-train \cite{oseledets2011tensor}.
Each of the decompositions exploits a different definition of tensor rank,
and can be used to recover tensors that are low rank in that sense, including
CP \cite{krishnamurthy2013low, jain2014provable, barak2016noisy, ashraphijuo2017fundamental, ghadermarzy2018learning, liu2020tensor}, Tucker \cite{gandy2011tensor, mu2014square, xia2017statistically, yokota2018missing, zhang2019cross, huang2020hosvd} and tensor-train \cite{wang2016tensor, yuan2018high}.
In this paper, we assume the tensor has approximately low multilinear rank, corresponding to a tensor that can be approximated by a low rank Tucker decomposition.

Existing techniques used for tensor completion include subspace projection onto unfoldings \cite{krishnamurthy2013low}, alternating minimization \cite{jain2014provable, liu2020tensor, wang2016tensor}, gradient descent \cite{yuan2018high} and expectation-maximization \cite{yokota2018missing}; 
different surrogates for the rank penalty have been used, including convex surrogates like nuclear norm on unfoldings \cite{tomioka2010estimation, gandy2011tensor, aidini20181} or specific flattenings \cite{mu2014square} and the maximum norm on factors \cite{ghadermarzy2018learning}, and nonconvex surrogates such as the minimum number of rank-1 sign tensor components \cite{ghadermarzy2018learning}.
We use a higher-order SVD (HOSVD) approach that does not require rank surrogates. 
The two methods closest to ours are \textsc{HOSVD\_w} \cite{huang2020hosvd}, which computes a weighted HOSVD by reweighting the tensor twice (before and after the HOSVD) by the inverse square root propensity, and the method in \cite{xia2017statistically}, which computes a HOSVD on a second-order estimator of a missing completely at random tensor, and we call it \textsc{SO-HOSVD}. 
We compare our method with \textsc{HOSVD\_w} and \textsc{SO-HOSVD} theoretically in Section~\ref{sec:error-in-tensor-completion-informal} and numerically in Section~\ref{sec:experiments-synthetic}.

Most previous works on tensor completion have used the assumption that entries are missing completely at random (MCAR).
The missing not at random (MNAR) setting is less studied,
especially for tensors.
A missingness pattern is MNAR when the observation probabilities (also called propensities) of different entries are not equal and may depend on the entry values themselves.
In the matrix MNAR setting, a popular observation model is 1-bit observations \cite{davenport20141}:
each entry is observed with a probability that comes from applying a differentiable function $\sigma: \RR \rightarrow [0, 1]$ to the corresponding entry in a parameter matrix, which is assumed to be low rank.
Two popular convex surrogates for the rank have been used to
estimate the parameter matrix from an entrywise binary observation pattern using a regularized likelihood approach: nuclear norm \cite{davenport20141, ma2019missing, aidini20181}
and max-norm \cite{cai2013max, ghadermarzy2018learning}.
We show that we can achieve a small propensity estimation error by solving for a parameter tensor with low multilinear rank using a (roughly) square flattening.
This approach outperforms simple slicing or flattening methods.

In this paper, we study the problem of provably completing a MNAR tensor with (approximately) low multilinear rank.
We use a two-step procedure to first estimate the propensities in this tensor and then predict missing values by HOSVD on the inverse propensity reweighted tensor. 
We give the error bound on final estimation as Theorem~\ref{thm:main-theorem-informal}.

This paper is organized as follows.
Section~\ref{sec:notations} sets up our notations.
Section~\ref{sec:problem_setting} formally describes the problem we tackle in this paper.
Section~\ref{sec:methodology} gives an overview of our algorithms.
Section~\ref{sec:error_analysis} gives further clarification and finite sample error bounds for the algorithms.
Section~\ref{sec:experiments} shows numerical experiments.

\section{Notations}
\label{sec:notations}
\paragraph{Basics}
We define $[N] = \{1, \ldots, N\}$ for a positive integer $N$.
Given a set $S$, we denote its cardinality by $|S|$.
$\subset$ denotes strict subset. 
We denote $f(n) = O(g(n))$ if there exists $C>0$ and $N$ such that $|f(n)| \leq C g(n)$ for all $n \geq N$. 
The indicator function $\mathbf{1}(X)$ has value 1 if the condition $X$ is true, and 0 otherwise.

\paragraph{Matrices and tensors}
We denote \textit{vector}, \textit{matrix}, and \textit{tensor} variables
respectively by lowercase letters ($x$), capital letters ($X$) and Euler script letters ($\T{X}$).
For a matrix $X \in \RR^{m \times n}$, $\sigma_1 (X) \geq \sigma_2 (X) \geq \cdots \geq \sigma_{\min \{m, n\}} (X)$ denote its singular values, $\norm{X}$ denotes its 2-norm, $\nucnorm{X}$ denotes its nuclear norm, $\mathrm{tr}(X)$ denotes its trace, and with another matrix $Y \in \RR^{m \times n}$, $\langle X, Y \rangle := \mathrm{tr}(X^\top Y)$ denotes the matrix inner product.
For a tensor $\T{X} \in \mathbb{R}^{I_1 \times I_2 \times \cdots \times I_N}$, $\maxnorm{\T{X}}$ denotes its entrywise maximum absolute value.
The \emph{order} of a tensor is the number of dimensions; matrices are order-two tensors.
Each dimension is called a \emph{mode}.
To denote a part of matrix or tensor, we use a colon to denote the mode that is not fixed:
given a matrix $A \in \mathbb{R}^{I \times J}$, $A_{i, :}$ and $A_{:, j}$ denote the $i$th row and $j$th column of $A$, respectively.
A \emph{fiber} is a one-dimensional section of a tensor $\T{X}$, defined by fixing every index but one; for example, a fiber of the order-3 tensor $\T{X}$ is $X_{:, j, k}$.
A \emph{slice} is an $(N-1)$-dimensional section of an order-$N$ tensor $\T{X}$: a slice of the order-3 tensor $\T{X}$ is $X_{:, :, k}$.
The \emph{size} of a mode is the number of slices along that mode: the $n$-th mode of $\T{X}$ has size $I_n$. 
A tensor is \emph{cubical} if every mode is the same size: $\T{X} \in \RR^{I \times I \times \cdots \times I}$.
The \emph{mode-$n$ unfolding} of $\T{X}$, denoted as $\T{X}^{(n)}$, is a matrix whose columns are the mode-$n$ fibers of $\T{X}$.
For example, given an order-3 tensor $\T{X} \in \mathbb{R}^{I \times J \times K}$, $\T{X}^{(1)} \in \mathbb{R}^{I \times (J \times K)}$.

\paragraph{Products}
We denote the \emph{$n$-mode product} of a tensor $\T{X} \in \mathbb{R}^{I_1 \times I_2 \times \cdots \times I_N}$ with a matrix $U \in \mathbb{R}^{J \times I_n}$ by $\T{X} \times_n U \in \mathbb{R}^{I_1 \times \cdots \times I_{n-1} \times J \times I_{n+1} \times \cdots \times I_N}$; the $(i_1, i_2, \dots, i_{n-1}, j, i_{n+1}, \dots, i_N)$-th entry is $\Sigma_{i_n = 1}^{I_n} x_{i_1 i_2 \cdots i_{n-1} i_n i_{n+1} \cdots i_N} u_{j i_n}$.
$\otimes$ denotes the Kronecker product. 
Given two tensors with the same shape, we use $\odot$ to denote their entrywise product.

\paragraph{Missingness}
Given a partially observed order-$N$ tensor $\T{X} \in \mathbb{R}^{I_1 \times \cdots \times I_N}$, we denote its observation pattern by $\Upomega \in \{0, 1\}^{I_1 \times \cdots \times I_N}$:  the \emph{mask tensor} of $\T{X}$. It is a binary tensor that denotes whether each entry of $\T{X}$ is observed or not.
$\Upomega$ has the same shape as $\T{X}$, with entry value 1 if the corresponding entry of $\T{X}$ is observed, and 0 otherwise.
With an abuse of notation, we call $\Omega := \{(i_1, i_2, \dots, i_N) | \Upomega_{i_1, i_2, \dots, i_N} = 1\}$ the \emph{mask set} of $\T{X}$.
Given a tensor $\T{X} \in \mathbb{R}^{I_1 \times I_2 \times \cdots \times I_N}$, we use $\T{E}(i_1, i_2, \dots, i_N)$ to denote a binary tensor with the same shape as $\T{X}$, with value 1 at the $(i_1, i_2, \dots, i_N)$-th entry and 0 elsewhere.

\paragraph{Square unfoldings}
Extending the notation of \cite{mu2014square}, with a matrix $A \in \mathbb{R}^{I \times J}$ and integers $I', J'$ satisfying $IJ = I' J'$, $\texttt{reshape} (A, I', J')$ gives an $I' \times J'$ matrix $A'$ with entries taken columnwise from $A$.
Given a tensor $\T{X} \in \mathbb{R}^{I_1 \times I_2 \times \cdots \times I_N}$, we can partition the indices of its $N$ modes into two sets, $S$ and $S^C = [N] \setminus S$, 
and permute the order of the $N$ modes by permutation $\pi_S = (S_1, \ldots, S_{|S|}, S^C_1, \ldots, S^C_{N - |S|})$, 
so that the modes in set $S$
appear first, followed by modes in $S^C$:
\[
\T{X}_{i_1 \cdots i_N} = \pi_S(\T{X})_{i_{S_1}} \cdots i_{S_{|S|}} i_{S^C_1} \cdots i_{S^C_{N - |S|}}.
\]
Denote the \emph{$S$-unfolding} of $\T{X}$ as 
\[
\T{X}_S := \texttt{reshape} (\pi_S(\T{X})^{(1)}, \prod_{n \in S} I_n, \prod_{n \in S^C} I_n).
\]
Our methods for tensor completion rely on methods for matrix completion that work best for square matrices. To make $\T{X}_S$ as square as possible, $\Bigl |\prod_{n \in S} I_n -  \prod_{n \in S^C} I_n\Bigr |$ should be as small as possible. Hence we define the \emph{square set} of $\T{X}$ as
\begin{equation}
	S_\square = \underset{S \subset [N]}{\text{arg min}} \Bigl |\prod_{n \in S} I_n -  \prod_{n \in S^C} I_n\Bigr |,
	\nonumber
\end{equation}
the \emph{square unfolding} of $\T{X}$ as
\begin{equation}
	\T{X}_{\square} := \texttt{reshape} (\pi_{S_\square} (\T{X})^{(1)}, \prod_{n \in S_\square} I_n, \prod_{n \in S^C_\square} I_n),
	\nonumber
\end{equation}
and the \emph{square norm} $\snorm{\T{X}} := \nucnorm{\T{X}_\square}$  of $\T{X}$.

\paragraph{Dimensions of unfoldings}
\label{sec:dims-of-unfoldings}
For brevity, we denote $I_S := \prod_{n \in S} I_n$, $I_{S^C} := \prod_{n \in S^C} I_n$, $I_\square := \prod_{n \in S_\square} I_n$, $I_{\square^C} := \prod_{n \in S_\square^C} I_n$, $ I_{(-n)}:= \prod_{m \in [N], m \neq n} I_m$.
Thus $I_{[N]} = \prod_{n \in [N]} I_n = I_S \cdot I_{S^C} = I_\square \cdot I_{\square^C} = I_n \cdot I_{(-n)}$.

\section{Problem setting}
\label{sec:problem_setting}
In this paper, we study the following problem:
given a partially observed tensor $\obs{\T{B}} \in \mathbb{R}^{I_1 \times \cdots \times I_N}$
with MNAR entries, how can we recover its missing values?

Throughout the paper, we denote the true order-$N$ tensor we want to complete as $\T{B} \in \mathbb{R}^{I_1 \times \cdots \times I_N}$.
For each $n \in [N]$, we suppose $I_n \leq I_{(-n)}:= \prod_{m \in [N], m \neq n} I_m$ for cleanliness. 
We assume there exists a propensity tensor $\T{P}  \in \mathbb{R}^{I_1 \times I_2 \times \cdots I_N}$, such that $\T{B}_{i_1 i_2 \cdots i_N}$ is observed with probability $\T{P}_{i_1 i_2 \cdots i_N}$.
We observe the entries without noise: with the observation pattern $\Upomega$, $\T{B}_\mathrm{obs} = \T{B} \odot \Upomega$.

A tensor $\T{B}$ has \emph{multilinear rank} $(r_1^\mathrm{true}, r_2^\mathrm{true}, \ldots, r_N^\mathrm{true})$ if $r_n^\mathrm{true}$ is the rank of $\T{B}^{(n)}$.
For any $n \in [N]$, $r_n^\mathrm{true} \leq I_n $.
We can write the \emph{Tucker decomposition} of the tensor $\T{B}$ as $\T{B} = \T{G}^\mathrm{true} \times_1 U_1^\mathrm{true} \times \cdots \times_N U_N^\mathrm{true}$, with \emph{core tensor} $\T{G}^\mathrm{true} \in \mathbb{R}^{r_1^\mathrm{true} \times \cdots \times r_N^\mathrm{true}}$ and column orthonormal \emph{factor matrices} $U_n^\mathrm{true} \in \mathbb{R}^{I_n \times r_n^\mathrm{true}}$ for $n \in [N]$.

We seek a \emph{fixed-rank approximation} of $\T{B}$ by a tensor with multilinear rank $(r_1, r_2, \cdots, r_N)$: we want to find a core tensor $\T{W} \in \mathbb{R}^{r_1 \times r_2 \times \cdots \times r_N}$ and $N$ factor matrices $Q_n \in \mathbb{R}^{I_n \times r_n}$, $n \in [N]$ with orthonormal columns, such that $\T{B} \approx \T{W} \times_1 Q_1 \times_2 \cdots \times_N Q_N$. We generally seek a low multilinear rank decomposition with $r_n < I_n $.

\section{Methodology}
\label{sec:methodology}
\begin{algorithm}[t]
\caption{\textsc{ConvexPE}: Convex Propensity Estimation}
\label{alg:propensity_estimation_provable}
\begin{algorithmic}[1]
	\Require{mask tensor $\Upomega \in \mathbb{R}^{I_1 \times \cdots \times I_N}$,
		link function $\sigma$,
		 thresholds $\tau$, $\gamma$}
	\Ensure{estimated propensity tensor $\widehat{\T{P}}$}
	\State Compute $S_\square$, the square set of $\Upomega$.
	\State Compute best completion
	$\widehat{\T{A}}_\square = \underset{\Gamma\in\mathcal{S}_{\tau,\gamma}}{\text{argmax}}\sum_{i=1}^{I_\square} \sum_{j=1}^{I_{\square^C}} [(\Upomega_\square)_{i j}\log \sigma(\Gamma_{i j})+ (1-(\Upomega_\square)_{i j})\log(1-\sigma(\Gamma_{i j}))]$,
	
	where $\mathcal{S}_{\tau,\gamma} = \big\{ \Gamma\in\mathbb{R}^{I_\square \times I_{\square^C}}: \nucnorm{\Gamma} \leq \tau\sqrt{I_{[N]}}, \maxnorm{\Gamma} \leq \gamma \big\}$.
	\State Estimate propensities $\widehat{\T{P}} = \sigma(\widehat{\T{A}})$.
	\State \Return{$\widehat{\T{P}}$}
\end{algorithmic}
\end{algorithm}

\begin{algorithm}[t]
	\caption{\textsc{NonconvexPE}: Nonconvex Propensity Estimation}
\label{alg:propensity_estimation_alt}
\begin{algorithmic}[1]
	\Require{mask tensor $\Upomega \in \mathbb{R}^{I_1 \times \cdots \times I_N}$, 
		link function $\sigma$, 
		step size $t$, 
		initialization $\{\bar{\T{G}}^{\T{A}}, \bar{U}_1^{\T{A}}, \ldots, \bar{U}_N^{\T{A}}\}$ (or target rank $(r_1, \ldots, r_N)$ with a certain initialization rule)}
	\Ensure{estimated propensity tensor $\widehat{\T{P}}$}
	\State Initialize core tensor and factor matrices $\T{G}^{\T{A}}, U_1^{\T{A}}, \ldots, U_N^{\T{A}} \gets \bar{\T{G}}^{\T{A}}, \bar{U}_1^{\T{A}}, \ldots, \bar{U}_N^{\T{A}}$.
	\State Define objective
	\begin{equation*}
		\begin{aligned}
			\hspace{1em} & f (\T{G}^{\T{A}}, \{U_n^{\T{A}}\}_{n \in [N]}) := \sum_{i_1 \cdots i_N} -\Upomega_{i_1 \cdots i_N} \log \sigma(\widehat{\T{A}}_{i_1 \cdots i_N}) -(1-\Upomega_{i_1 \cdots i_N})\log(1-\sigma(\widehat{\T{A}}_{i_1 \cdots i_N})),\\
			& \text{in which } \widehat{\T{A}} = \T{G}^{\T{A}} \times_1 U_1^{\T{A}} \times_2 \cdots \times_N U_N^{\T{A}}.
		\end{aligned}
	\end{equation*}
	\Do
	\State Compute gradients with respect to core tensor and factor matrices
	$\left(\frac{\partial f}{\partial \T{G}^{\T{A}}}, 
	\frac{\partial f}{\partial U_1^{\T{A}}},\ldots,\frac{\partial f}{\partial U_N^{\T{A}}}\right)$.
	\State Perform gradient descent update: $\T{G}^{\T{A}}, U_1^{\T{A}}, \ldots, U_N^{\T{A}} \gets \left(\T{G}^{\T{A}} - t \frac{\partial f}{\partial \T{G}^{\T{A}}}, U_1^{\T{A}} - t \frac{\partial f}{\partial U_1^{\T{A}}}, \ldots, U_N^{\T{A}}- t \frac{\partial f}{\partial U_N^{\T{A}}}\right)$.
	\doWhile{not converged}
	\State Estimate propensities $\widehat{\T{P}} = \sigma(\widehat{\T{A}})$.
	\State \Return{$\widehat{\T{P}}$}
\end{algorithmic}
\end{algorithm}

\begin{algorithm}[t]
	\caption{\textsc{TenIPS}: Tensor completion by Inverse Propensity Sampling}
	\label{alg:tensor_completion}
	\begin{algorithmic}[1]
		\Require{mask set $\Omega$, partially observed tensor $\T{B}_\text{obs} \in \mathbb{R}^{I_1 \times \cdots \times I_N}$, propensity tensor $\T{P} \in \mathbb{R}^{I_1 \times \cdots \times I_N}$, target rank $(r_1, r_2, \cdots, r_N)$}
		\Ensure{estimated tensor $\widehat{\T{X}}(\T{P})$}
		\State
		$ \bar{\T{X}}(\T{P}) \gets \sum_{(i_1, i_2, \ldots, i_N) \in \Omega} \frac{1}{\T{P}_{i_1 i_2 \cdots i_N}} \T{B}_\text{obs} \odot \T{E}(i_1, i_2, \dots, i_N)$
		\For{$n=1, 2, \ldots, N$}
		\Comment{Recover factors}
		\State $Q_n(\T{P}) \gets \text{left $r_n$ singular vectors of } \bar{\T{X}}(\T{P})^{(n)}$
		\EndFor
		\State $\T{W} \leftarrow \bar{\T{X}}(\T{P}) \times_1 Q_1(\T{P})^\top \times_2 \cdots \times_N Q_N(\T{P})^\top$
		\Comment{Recover core}
		\State $\widehat{\T{X}}(\T{P}) \gets \T{W} (\T{P}) \times_1 Q_1 (\T{P}) \times_2 \cdots \times_N Q_N(\T{P})$
		\State \Return{$\widehat{\T{X}}(\T{P})$}
	\end{algorithmic}
\end{algorithm}

\begin{table}[t]
	\centering
	\caption{Propensity estimation algorithms.}	
	\begin{tabular}{lll}
		\toprule
		~ & base algorithm & hyperparameters \\
		\midrule
		\textsc{ConvexPE} (Algorithm~\ref{alg:propensity_estimation_provable}) & proximal-proximal-gradient & $\tau$ and $\gamma$ \\
		\textsc{NonconvexPE} (Algorithm~\ref{alg:propensity_estimation_alt}) & gradient descent & step size $t$ and target rank\\
		\bottomrule
	\end{tabular}
	\label{table:PE_algorithms}
\end{table}

Our algorithm proceeds in two steps.
First, we estimate the propensities by \textsc{ConvexPE} (Algorithm~\ref{alg:propensity_estimation_provable}) or \textsc{NonconvexPE} (Algorithm~\ref{alg:propensity_estimation_alt}), with an overview in Table~\ref{table:PE_algorithms}.
Both of these algorithms use a Bernoulli maximum likelihood estimator for 1-bit matrix completion \cite{davenport20141} to estimate the propensities from the mask tensor $\Upomega$, aiming to recover propensities that come from the low rank parameters.
\textsc{ConvexPE} explicitly requires the propensities to be neither too large or too small. 
\textsc{NonconvexPE} does not require the associated tuning parameters, but empirically returns a good solution if the true propensity tensor $\T{P}$ has this property.
With the estimated propensity tensor $\widehat{\T{P}}$, we estimate the data tensor $\T{B}$ by \textsc{TenIPS} (Algorithm~\ref{alg:tensor_completion}), a procedure that only requires a Tucker decomposition on the propensity-reweighted observations.

Our propensity estimation uses the observation model of 1-bit matrix completion. 
Each entry of $\T{P}$ comes from applying a differentiable link function $\sigma: \RR \rightarrow [0, 1]$ to the corresponding entry of a parameter tensor $\T{A}$, which we are trying to solve.
An instance is the logistic function $\sigma(x)=1/(1+e^{-x})$.
We assume $\T{A}$ has low multilinear rank.
In \textsc{ConvexPE} (Algorithm~\ref{alg:propensity_estimation_provable}), $\T{A}_\square$ is low-rank from Lemma~\ref{lem:tucker_unfolding}.
We also assume an upper bound on the nuclear norm of $\T{A}_\square$, a convex surrogate for its low-rank property.
\textsc{ConvexPE} can be implemented by the proximal-proximal-gradient method (PPG) \cite{ryu2017proximal} or the proximal alternating gradient descent.
In Section~\ref{sec:error_analysis}, we will show that on a square tensor, the square unfolding achieves the smallest upper bound for propensity estimation among all possible unfoldings.

In practice, the square unfolding of a tensor is often a large matrix: $I^{N/2}$-by-$I^{N/2}$ for a cubical tensor with order $N$. 
Since each iteration of the PPG subroutine in \textsc{ConvexPE} requires the computation of a truncated SVD, this algorithm becomes too expensive in such case.
Also, it does not make full use of the low multilinear rank property of $\T{A}$. 
As a substitute, we propose \textsc{NonconvexPE} (Algorithm~\ref{alg:propensity_estimation_alt}), which uses gradient descent (GD) on the core tensor $\T{G}^{\T{A}}$ and factor matrices $\{U_n^{\T{A}}\}_{n \in [N]}$ to minimize the objective function $f (\T{G}^{\T{A}}, \{U_n^{\T{A}}\}_{n \in [N]})$ defined in Line~4. 
It achieves a feasible solution with similar quality as \textsc{ConvexPE}, and does not require the tuning of thresholds $\tau$ and $\gamma$. 
This can be attributed to the fact that the objective function $f$ is multi-convex with respect $(\T{G}^{\T{A}}, U_1^{\T{A}}, \ldots, U_N^{\T{A}})$. 
The gradient computation can be found in Appendix~\ref{sec:gradients}. 

\textsc{TenIPS} (Algorithm~\ref{alg:tensor_completion}) completes the observed data tensor $\T{B}_\mathrm{obs}$ by HOSVD on its entrywise inverse propensity reweighting $\bar{\T{X}}(\T{P})$, as defined in Line~2. 
For each $(i_1, i_2, \dots, i_N) \in \Omega$, the corresponding term in $\bar{\T{X}}(\T{P})$ is an unbiased estimate for $ \T{B} \odot \T{E}(i_1, i_2, \dots, i_N)$: 
\begin{equation}
	\begin{aligned}
		\Expect [\frac{1}{\T{P}_{i_1 i_2 \cdots i_N}} \T{B}_\text{obs} \odot \T{E}(i_1, i_2, \dots, i_N)] & = \T{P}_{i_1 i_2 \cdots i_N} \cdot \frac{1}{\T{P}_{i_1 i_2 \cdots i_N}} \T{B}_\text{obs} \odot \T{E}(i_1, i_2, \dots, i_N)\\
		& = \T{B} \odot \T{E}(i_1, i_2, \dots, i_N),
	\end{aligned}
	\nonumber
\end{equation}
in which the second equality comes from noiseless observations.
Thus $\bar{\T{X}}(\T{P})$ is an unbiased estimator for $\T{B}$: 
\begin{equation}
	\begin{aligned}
		\Expect \bar{\T{X}}(\T{P}) & = \Expect \Big[\sum_{(i_1, i_2, \dots, i_N) \in \Omega} \frac{1}{\T{P}_{i_1 i_2 \cdots i_N}} \T{B}_\text{obs} \odot \T{E}(i_1, i_2, \dots, i_N)\Big]\\
		& = \sum_{i_1=1}^{I_1} \sum_{i_2=1}^{I_2} \cdots \sum_{i_N=1}^{I_N} \Expect \Big[\frac{1}{\T{P}_{i_1 i_2 \cdots i_N}} \T{B}_\text{obs} \odot \T{E}(i_1, i_2, \dots, i_N) \Big]\\
		& = \sum_{i_1=1}^{I_1} \sum_{i_2=1}^{I_2} \cdots \sum_{i_N=1}^{I_N} \T{B} \odot \T{E}(i_1, i_2, \dots, i_N) = \T{B}.
	\end{aligned}
	\nonumber
\end{equation}
The input propensity tensor can be either true ($\T{P}$) or estimated ($\widehat{\T{P}}$). 
With the estimated propensity $\widehat{\T{P}}$, we get $\bar{\T{X}}(\widehat{\T{P}})$ instead of $\bar{\T{X}}(\T{P})$, $\widehat{\T{X}}(\widehat{\T{P}})$ instead of $\widehat{\T{X}}(\T{P})$;
for brevity, we denote $\bar{\T{X}}(\T{P})$ and $\T{\Xh}(\T{P})$ by $\bar{\T{X}}$ and $\T{\Xh}$, respectively.
We show the estimation error for $\T{B}$ in Theorem~\ref{thm:main-theorem-informal}.

\section{Error analysis}
\label{sec:error_analysis}

To bound the relative estimation error $\fnorm{\widehat{\T{X}}(\widehat{\T{P}}) - \T{B}} / \fnorm{\T{B}}$,
we first bound the error in the propensity estimates in \textsc{ConvexPE},
and then consider how this error propagates into the error of our final tensor estimate in \textsc{TenIPS}.
Theorem~\ref{thm:square_unfolding_for_general_matrices} shows the optimality of the square unfolding for propensity estimation; Theorem~\ref{thm:main-theorem-informal} presents a special case of our bound on the tensor completion error with estimated propensities, with the full version as Appendix~\ref{sec:error-in-tensor-completion-formal}, Theorem~\ref{thm:main-theorem-formal}.
We defer their proofs to Appendix~\ref{sec:proof-for-main-theorem}.

\subsection{Error in propensity estimates}
\label{sec:propensity-estimation}
We first show a corollary of \cite[Lemma 6 (2) and Lemma 7]{mu2014square} that bounds the rank of an unfolding.
\begin{lemma}
	\label{lem:tucker_unfolding}
	Suppose $\T{X}$ has Tucker decomposition
	$\T{X}=\T{C}\times_1 U_1 \times_2 U_2 \times_3 \cdots \times_N U_N$,
	where $\T{C} \in \RR^{r_1^\mathrm{true} \times r_2^\mathrm{true} \times \cdots \times r_N^\mathrm{true}}$
	and $U_n \in \RR^{I_n \times r_n^\mathrm{true}}$ for $n \in [N]$.
	Given $S \subset [N]$, $\T{X}_S = \underset{j \in S}{\otimes} U_j \cdot \T{C}_S \cdot \Big(\underset{j \in [N] \backslash S}{\otimes} U_j\Big)^\top$, and thus $\mathrm{rank}( \T{X}_S) \leq \min \Bigl\{\prod_{n \in S} r_n^\mathrm{true}, \, \prod_{n \in [N] \backslash S} r_n^\mathrm{true} \Big\}$.
\end{lemma}

\begin{proof}
	\cite[Lemma 6 (2)]{mu2014square} states that
	\begin{equation}
		\label{eq:tucker_unfolding_original}
		\T{X}_{[n]} = (U_n \otimes U_{n-1} \otimes \cdots \otimes U_1) \,\T{C}_{[n]}\, (U_N \otimes U_{N-1} \otimes \cdots \otimes U_{n+1})^\top
	\end{equation}
	for $n \in [N]$.
	Thus Lemma 1 holds for $\T{X}$ by applying
	Equation~\ref{eq:tucker_unfolding_original} to an entry-reordered tensor
	$\widetilde{\T{X}} \in \RR^{I_{n_1} \times I_{n_2} \times \cdots \times I_{n_N}}$,
	such that $S = \{n_j\}_{j \in [S]}$ and $\widetilde{\T{X}}_{i_{n_1} i_{n_2} \cdots i_{n_N}} = \T{X}_{i_1 i_2 \cdots i_N}$.
	The upper bound for $\mathrm{rank}( \T{X}_S)$ follows.
\end{proof}

As a corollary of \cite[Lemma 1]{davenport20141} and \cite[Theorem 2]{ma2019missing}, we have Lemma~\ref{lem:propensity_error} for the Frobenius norm error of the propensity tensor estimate.
\begin{lemma}
	\label{lem:propensity_error}
	Assume that $\T{P} = \sigma (\T{A})$.
	Given a set $S \subset [N]$, together with the following assumptions:
	\begin{itemize}[leftmargin=2em,topsep=0pt,partopsep=1ex,parsep=0ex]
		\item[\textbf{A1.}] $\T{A}_S$ has bounded nuclear norm: there exists a constant
		$\theta > 0 $ such that $\nucnorm{\T{A}_S} \leq \theta \sqrt{I_{[N]}}$.
		\item[\textbf{A2.}] Entries of $\T{A}$ have bounded absolute value: there exists a constant $\alpha>0$ such that $\maxnorm{\T{A}} \leq \alpha$.
	\end{itemize}
	Suppose we run \textsc{ConvexPE} (Algorithm~\ref{alg:propensity_estimation_provable}) with thresholds satisfying $\tau \geq \theta$ and $\gamma \geq \alpha$ to obtain an estimate $\T{\Ph}$ of $\T{P}$.
	With $L_\gamma := \sup_{x\in[-\gamma,\gamma]} \frac{|\sigma'(x)|}{\sigma(x)(1-\sigma(x))}$, there exists a universal constant $C > 0$ such that if $I_S + I_{S^C} \geq C$, with probability at least $1 - \frac{C}{I_S + I_{S^C}}$, the estimation error
	\begin{equation}
		\label{eq:propensity_error_bound}
		\frac{1}{I_{[N]}} \fnorm{\widehat{\T{P}} - \T{P}}^2
		\leq 4 e L_\gamma \tau \Big(\frac{1}{\sqrt{I_S}}+\frac{1}{\sqrt{I_{S^C}}}\Big).
	\end{equation}
\end{lemma}

In the simplest case, when $N$ is an even integer and $I_1 = \ldots = I_N = I$, the right-hand side (RHS) of the estimation error in Equation~\ref{eq:propensity_error_bound} is in $O(I^{N/4})$.

Theorem~\ref{thm:square_unfolding_for_general_matrices} then shows that the square unfolding achieves the smallest upper bound for propensity estimation among all possible unfoldings sets $S$.
\begin{theorem}
	\label{thm:square_unfolding_for_general_matrices}
	Instate the same conditions as Lemma~\ref{lem:propensity_error},
	and further assume that there exists a constant $c>0$ such that $r_n^\mathrm{true} \leq c I_n$ for every $n \in [N]$.
	Then $S = S_\square$ gives the smallest upper bound (RHS of the result in Lemma~\ref{lem:propensity_error}) on the propensity estimation error $\fnorm{\widehat{\T{P}} - \T{P}}^2$ among all unfolding sets $S \subset [N]$.
\end{theorem}

\begin{proof}
	Denote $r_S^\mathrm{true} := \prod_{n \in S} r_n^\mathrm{true}$ and $r_{S^C}^\mathrm{true} := \prod_{n \in [N] \backslash S} r_n^\mathrm{true}$. 
	We know from Lemma~\ref{lem:tucker_unfolding} of the main paper that for every unfolding of $\T{A}$, $\rank{\T{A}_S} \leq \min \Bigl\{r_S^\mathrm{true}, \, r_{S^C}^\mathrm{true} \Big\}$. 
	Since $\nucnorm{\T{A}_S} \leq \sqrt{\rank{\T{A}_S}} \cdot \fnorm{\T{A}} \leq \alpha \sqrt{\rank{\T{A}_S}} \cdot I^\frac{N}{2}$, we need $\tau \geq \theta \geq \alpha \sqrt{\rank{\T{A}_S}}$ for the conditions of Lemma~\ref{lem:propensity_error} to hold.
	For simplicity, suppose $\tau = \alpha \sqrt{\rank{\T{A}_S}}$, the smallest possible value for the exact recovery of $\T{A}$.
	
	Without loss of generality, suppose $|S| \leq \frac{N}{2}$.
	We have
	\begin{equation}
		\begin{aligned}
			\frac{1}{I_{[N]}} \fnorm{\widehat{\T{P}} - \T{P}}^2 & \leq 4 e L_\gamma \tau \Big(\frac{1}{\sqrt{I_S}}+\frac{1}{\sqrt{I_{S^C}}}\Big) \\
			& = 4 e L_\gamma \alpha \cdot \sqrt{\rank{\T{A}_S}} \Big(\frac{1}{\sqrt{I_S}}+\frac{1}{\sqrt{I_{S^C}}}\Big) \\
			& \leq 4 e L_\gamma \alpha \Big(\sqrt{\frac{r_S^\mathrm{true}}{I_S}}+\sqrt{\frac{r_{S^C}^\mathrm{true}}{I_{S^C}}}\Big) \\
			& \leq 4 e L_\gamma \alpha \Big(\sqrt{c^{|S|}}+ \sqrt{c^{N - |S|}}\Big).
		\end{aligned}
		\nonumber
	\end{equation}
	The final expression is the smallest when $S = S_\square$.
\end{proof}

\subsection{Error in tensor completion: special case}
\label{sec:error-in-tensor-completion-informal}
We present a special case of our bound on the recovery error for a cubical tensor with equal multilinear rank as Theorem~\ref{thm:main-theorem-informal}.
This bound is dominated by the error from the matrix Bernstein inequality \cite{tropp2015introduction} on each of the $N$ unfoldings, and asymptotically goes to 0 when the tensor size $I \rightarrow \infty$. 
Note that our full theorem applies to any tensor; we defer the formal statement to Appendix~\ref{sec:error-in-tensor-completion-formal}, Theorem~\ref{thm:main-theorem-formal}. 

\begin{theorem}
	\label{thm:main-theorem-informal}
	Consider an order-$N$ cubical tensor $\T{B}$ with size $I_1 = \cdots = I_N = I$ and multilinear rank $r_1^\mathrm{true} = \cdots = r_N^\mathrm{true} = r < I$, and two order-$N$ cubical tensors $\T{P}$ and $\T{A}$ with the same shape as $\T{B}$.
	Each entry of $\T{B}$ is observed with probability from the corresponding entry of $\T{P}$.
	Assume $I \geq r N \log I$, and there exist constants $\psi, \alpha \in (0,\infty)$ such that $\maxnorm{\T{A}} \leq \alpha$, $\maxnorm{\T{B}} = \psi$.
	Further assume that for each $n \in [N]$, the condition number $\frac{\sigma_{1} (\T{B}^{(n)})}{\sigma_r (\T{B}^{(n)})} \leq \kappa$ is a constant independent of tensor sizes and dimensions. Then under the conditions of Lemma~\ref{lem:propensity_error}, with probability at least $1 - I^{-1}$, the fixed multilinear rank $(r, r, \ldots, r)$ approximation $\widehat{\T{X}}(\widehat{\T{P}})$ computed from \textsc{ConvexPE} and \textsc{TenIPS} (Algorithms~\ref{alg:propensity_estimation_provable} and \ref{alg:tensor_completion}) with thresholds $\tau \geq \theta$ and $\gamma \geq \alpha$ satisfies
	\begin{equation}
		\label{eq:error-bound-informal}
		\frac{\fnorm{\widehat{\T{X}}(\widehat{\T{P}}) - \T{B}}}{\fnorm{\T{B}}} \leq C N \sqrt{\frac{r \log I}{I}},
	\end{equation}
	in which $C$ depends on $\kappa$.
\end{theorem}

Note how this bound compares with the bounds for similar algorithms.
\textsc{HOSVD\_w} has a relative error of $O(r N^2 I^{-1/2} \log I)$ \cite[Theorem 3.3]{huang2020hosvd} for noiseless recovery with known propensities, and \textsc{SO-HOSVD} achieves a better bound of $O(\sqrt{\frac{r \log I}{I^{N-1}}})$ \cite[Theorem 3]{xia2017statistically} but assumes that the tensor is MCAR. 
In contrast, our bound holds for the tensor MNAR setting and does not require known propensities.
It is the first bound in this setting, to our knowledge.

We show a sketch of the proof for Theorem~\ref{thm:main-theorem-informal} to illustrate the main idea.
This is the special case of the full proof in Appendix~\ref{sec:proof-for-main-theorem}.

\begin{proof}(sketch)
	
	The propensity estimation error in Lemma~\ref{lem:propensity_error}, Equation~\ref{eq:propensity_error_bound} propagates to the error between $\bar{\T{X}}(\widehat{\T{P}})$ and $\bar{\T{X}}(\T{P})$: 
	\begin{equation}
		\label{eq:x_bar_error_from_propensity}
		\begin{aligned}
			& \fnorm{\bar{\T{X}}(\T{\Ph}) - \bar{\T{X}}}^2 \\
			& = \sum_{({i_1, i_2, \dots, i_N}) \in \Omega} \T{B}_{i_1 i_2 \cdots i_N}^2 (\frac{1}{\T{P}_{i_1 i_2 \cdots i_N}} - \frac{1}{\T{\Ph}_{i_1 i_2 \cdots i_N}})^2 \\
			& \leq  \psi^2 \sum_{({i_1, i_2, \dots, i_N}) \in \Omega} \Big(\frac{\T{P}_{i_1 i_2 \cdots i_N} - \T{\Ph}_{i_1 i_2 \cdots i_N}}{\T{P}_{i_1 i_2 \cdots i_N} \T{\Ph}_{i_1 i_2 \cdots i_N}} \Big)^2 \\
			& \leq  \frac{\psi^2}{\sigma(-\gamma)^2 \sigma(-\alpha)^2} \sum_{({i_1, i_2, \dots, i_N}) \in \Omega} \Big(\T{P}_{i_1 i_2 \cdots i_N} - \T{\Ph}_{i_1 i_2 \cdots i_N} \Big)^2 \\
			& \leq \frac{4 e L_\gamma \tau \psi^2}{\sigma(-\gamma)^2 \sigma(-\alpha)^2} \Big(\frac{1}{\sqrt{I_S}}+\frac{1}{\sqrt{I_{S^C}}}\Big) I_{[N]}.
		\end{aligned}
	\end{equation}
	The second inequality comes from $ \T{\Ph}_{i_1 i_2 \cdots i_N} \geq \sigma(-\gamma)$ and $ \T{P}_{i_1 i_2 \cdots i_N} \geq \sigma(-\alpha)$; the last inequality follows Lemma~\ref{lem:propensity_error}.
	
	Then on each of the $N$ unfoldings, the error of $\bar{\T{X}}^{(n)}(\widehat{\T{P}})$ from $\T{B}^{(n)}$
	\begin{equation}
		\label{eq:general_sampling_estimator_error}
		\begin{aligned}
			\norm{\bar{\T{X}}^{(n)}(\T{\Ph}) - \T{B}^{(n)}} & \leq \norm{\bar{\T{X}}^{(n)}(\T{\Ph}) - \bar{\T{X}}^{(n)}} + \norm{\bar{\T{X}}^{(n)} - \T{B}^{(n)}} \\
			& \leq \fnorm{\bar{\T{X}}^{(n)}(\T{\Ph}) - \bar{\T{X}}^{(n)}} + \norm{\bar{\T{X}}^{(n)} - \T{B}^{(n)}},\\
		\end{aligned}
	\end{equation}
	in which the first term can be bounded by Equation~\ref{eq:x_bar_error_from_propensity}, and the second term can be bounded by applying the matrix Bernstein inequality \cite{tropp2015introduction} to the sum of $[\frac{\mathbf{1}(\T{B}_{i_1 \cdots i_N} \text{ is observed})}{\T{P}_{i_1 \cdots i_N}} \T{B}_\text{obs} - \T{B}] \odot \T{E}(i_1, \dots, i_N)$ over all entries.
	
	The tensor $\bar{\T{X}}(\T{\Ph}) - \T{B}$ is often full-rank; if we directly use $\sqrt{I} \cdot \norm{\bar{\T{X}}^{(n)}(\T{\Ph}) - \T{B}^{(n)}}$ to bound $\fnorm{\bar{\T{X}}(\T{\Ph}) - \T{B}}$, the final error bound we get would be $O(N \sqrt{\log I})$, which increases with the increase of tensor size $I$. 
	Instead, we use the information of low multilinear rank to form the estimator $\T{\Xh} (\T{\Ph})$.
	In \textsc{TenIPS} (Algorithm~\ref{alg:tensor_completion}),
	\begin{equation}
		\begin{aligned}
			\T{\Xh} (\T{\Ph})& = [\bar{\T{X}}(\T{\Ph}) \times_1 Q_1^\top \times_2 \cdots \times_N Q_N^\top] \times_1 Q_1 \times_2 \cdots \times_N Q_N \\
			& = \bar{\T{X}}(\T{\Ph}) \times_1 Q_1Q_1^\top \times_2 \cdots \times_N Q_N Q_N^\top,
		\end{aligned}
		\nonumber
	\end{equation}
	in which each $Q_n$ is the column space of $\bar{\T{X}}^{(n)}(\T{\Ph})$.
	This projects each unfolding of $\bar{\T{X}} (\T{\Ph})$ onto its truncated column space. 
	By adding and subtracting $ \T{B} \times_1 Q_1Q_1^\top \times_2 \cdots \times_N Q_N Q_N^\top$, the projection of $\T{B}$ onto the column space of $\bar{\T{X}}(\T{\Ph})$ in each mode, we get the estimation error
	\begin{equation}
		\begin{aligned}
			& \fnorm{\T{\Xh}(\T{\Ph}) - \T{B}}^2 \\
			& =  \fnorm{\bar{\T{X}}(\T{\Ph}) \times_1 Q_1Q_1^\top \times_2 \cdots \times_N Q_N Q_N^\top - \T{B}}^2 \\
			& = \fnorm{(\bar{\T{X}}(\T{\Ph})-\T{B}) \times_1 Q_1Q_1^\top \times_2 \cdots \times_N Q_N Q_N^\top}^2\\
			& \quad +  \fnorm{\T{B} \times_1 Q_1Q_1^\top \times_2 \cdots \times_N Q_N Q_N^\top - \T{B}}^2\\
			& \quad +2 \langle (\bar{\T{X}}(\T{\Ph})-\T{B}) \times_1 Q_1Q_1^\top \times_2 \cdots \times_N Q_N Q_N^\top, \; \T{B} \times_1 Q_1Q_1^\top \times_2 \cdots \times_N Q_N Q_N^\top - \T{B} \rangle. \\
		\end{aligned}
		\nonumber
	\end{equation}
	On the RHS, the third term is an inner product of two mutually orthogonal tensors and is thus 0.
	The first term is low multilinear rank, and thus can be bounded by the spectral norm of the unfoldings as 
	\begin{equation}
		\begin{aligned}
			& \fnorm{(\bar{\T{X}}(\T{\Ph})-\T{B}) \times_1 Q_1Q_1^\top \times_2 \cdots \times_N Q_N Q_N^\top}^2 \\
			& \leq \min_{n \in [N]} \Big\{\fnorm{Q_n Q_n^\top (\bar{\T{X}}^{(n)}(\T{\Ph}) - \T{B}^{(n)})}^2 \Big\}\\
			& \leq \min_{n \in [N]} \Big\{r \cdot \norm{\bar{\T{X}}^{(n)}(\T{\Ph}) - \T{B}^{(n)}}^2 \Big\}.
		\end{aligned}
		\nonumber
	\end{equation}
	The second term can be bounded by the sum of squares of residuals $\sum_{n \in [N]} \fnorm{\T{B}\times_{n} (I - Q_{n} Q_{n}^\top)}^2$. 
	Each of the summand here is the perturbation error of $Q_n$ on the column space of $\T{B}^{(n)}$, and thus can be bounded by the Davis-Kahan sin($\Theta$) Theorem \cite{davis1970rotation, wedin1972perturbation, yu2015useful}. 
\end{proof}

\section{Experiments}
\label{sec:experiments}
All the code is in the GitHub repository at \url{https://github.com/udellgroup/TenIPS}.
We ran all experiments on a Linux machine with Intel\textsuperscript{\textregistered} Xeon\textsuperscript{\textregistered} E7-4850 v4 2.10GHz CPU and 1056GB memory, and used the logistic link function $\sigma(x)=1/(1+e^{-x})$ throughout the experiments.

We use both synthetic and semi-synthetic data for evaluation.
We first compare the propensity estimation performance under square and rectangular (along a specific mode) unfoldings, and then compare the tensor recovery error under different approaches.
For both propensity estimation and tensor completion, the relative error is defined as $\fnorm{\T{\Th} - \T{T}} / \fnorm{\T{T}}$, in which $\T{\Th}$ is the predicted tensor and $\T{T}$ is the true tensor.

There are four algorithms similar to \textsc{TenIPS} for tensor completion in our experiments: \textsc{SqUnfold}, which performs SVD to seek the low-rank approximation of the square unfolding of the propensity-reweighted $\T{B}_\mathrm{obs}$; \textsc{RectUnfold}, which applies SVD to the unfolding of the propensity-reweighted $\T{B}_\mathrm{obs}$ along a specific mode; \textsc{HOSVD\_w} \cite{huang2020hosvd} and \textsc{SO-HOSVD} \cite{xia2017statistically}. 
The popular nuclear-norm-regularized least squares \textsc{LstSq} that seeks $\hat{B} = \underset{X}{\arg \min} \sum_{(i, j) \in \Omega} (B_{ij}  - X_{ij})^2 + \lambda \nucnorm{X}$ on an unfolding of $\T{B}_\mathrm{obs}$ takes much longer to finish in our experiments and is thus prohibitive in tensor completion practice, so we omit it from most of our results. 

\subsection{Synthetic data}
\label{sec:experiments-synthetic}
We have the following observation models for synthetic tensors:
\begin{itemize}[leftmargin=5em,topsep=0pt,partopsep=1ex,parsep=0ex]
	\item[\textbf{Model A.}]
	\label{obs:syn-mcar}
	MCAR. The propensity tensor $\T{P}$ has all equal entries.
	\item[\textbf{Model B.}]
	\label{obs:syn-mnar-product}
	MNAR with an approximatly low multilinear rank parameter tensor $\T{A}$. 
	One special case is that $\T{A}$ is proportional to $\T{B}$: a larger entry is more likely to be observed.
\end{itemize}

In the first experiment, we evaluate the performance of propensity estimation on synthetic tensors with approximately low multilinear rank.
In each of the above observation models, we generate synthetic cubical tensors with equal size on each mode, and predict the propensity tensor by estimating the parameter tensor on either the square or rectangular unfolding.
\begin{figure}[t]
	\centering
	\begin{minipage}[b]{0.46\linewidth}
		\begin{subfigure}[t]{.45\linewidth}
		\includegraphics[width=\linewidth]{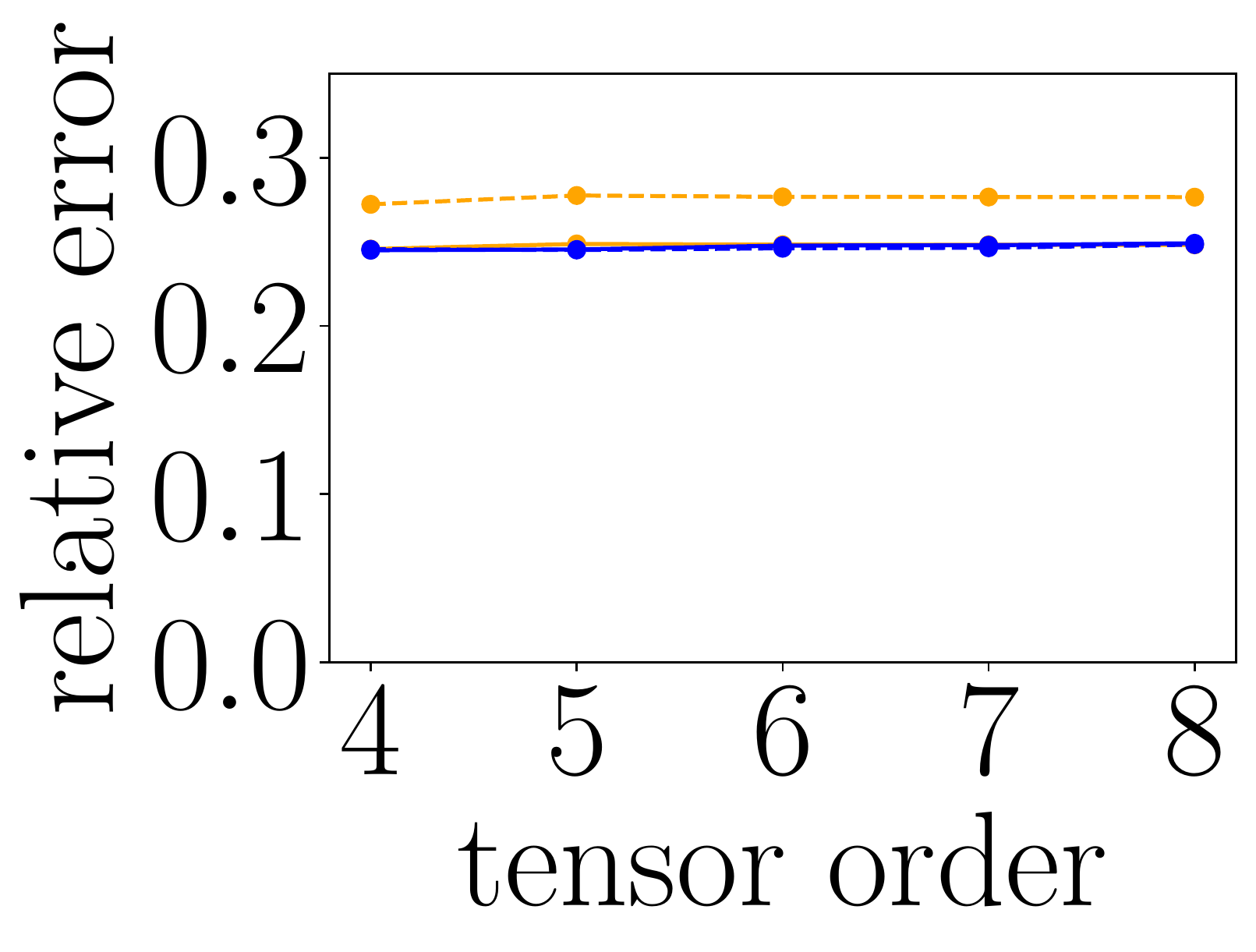}
		\caption{Model A}
		\label{fig:MCAR_relative_error_on_propensity}
	\end{subfigure}	
	\hspace{.05\linewidth}
			\begin{subfigure}[t]{.46\linewidth}
		\includegraphics[width=\linewidth]{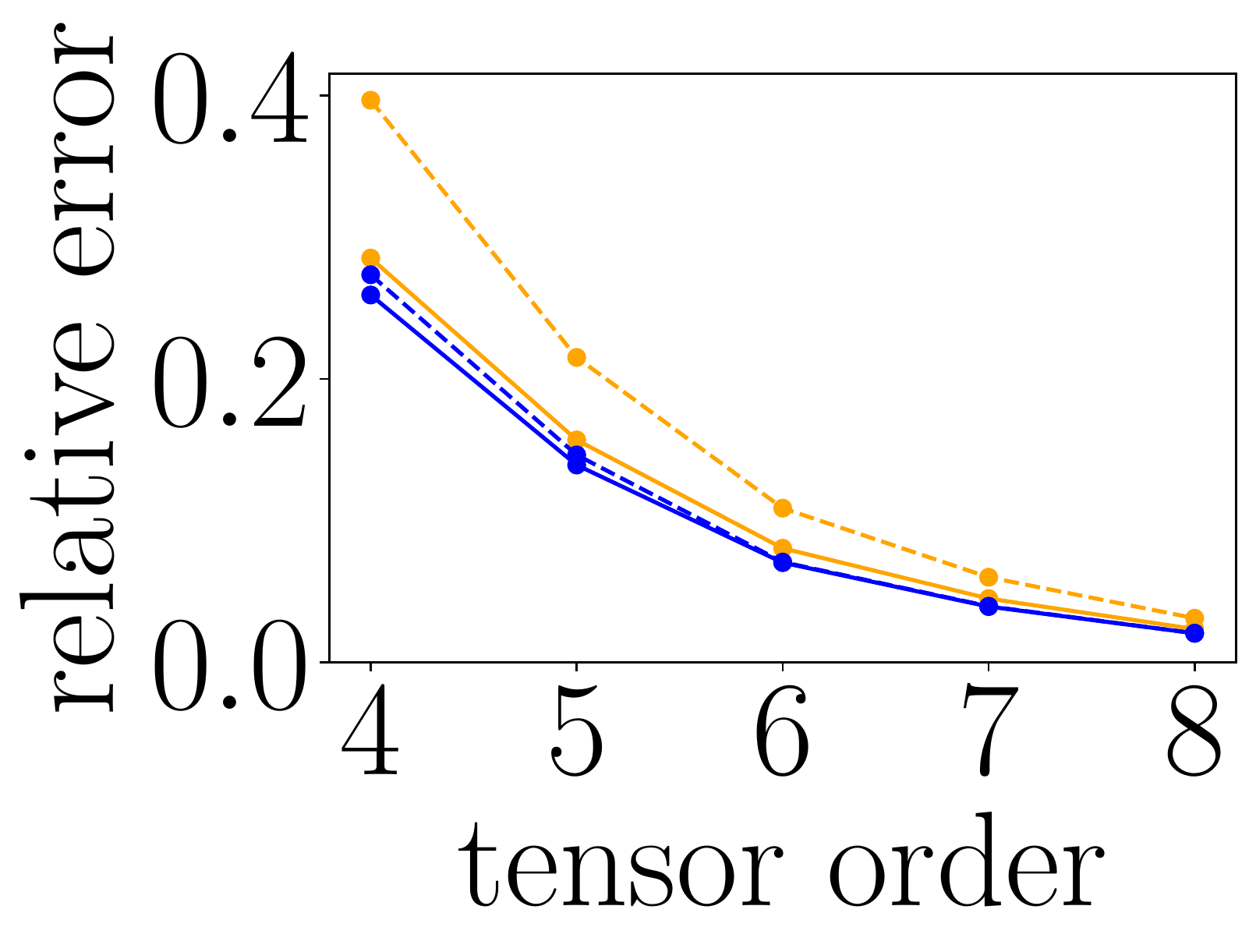}
		\caption{Model B}
	\label{fig:MNAR_relative_error_on_propensity}
	\end{subfigure}	
\end{minipage}
\begin{minipage}[b]{.42\linewidth}
	\stackunder[1pt]{\includegraphics[width=\linewidth]{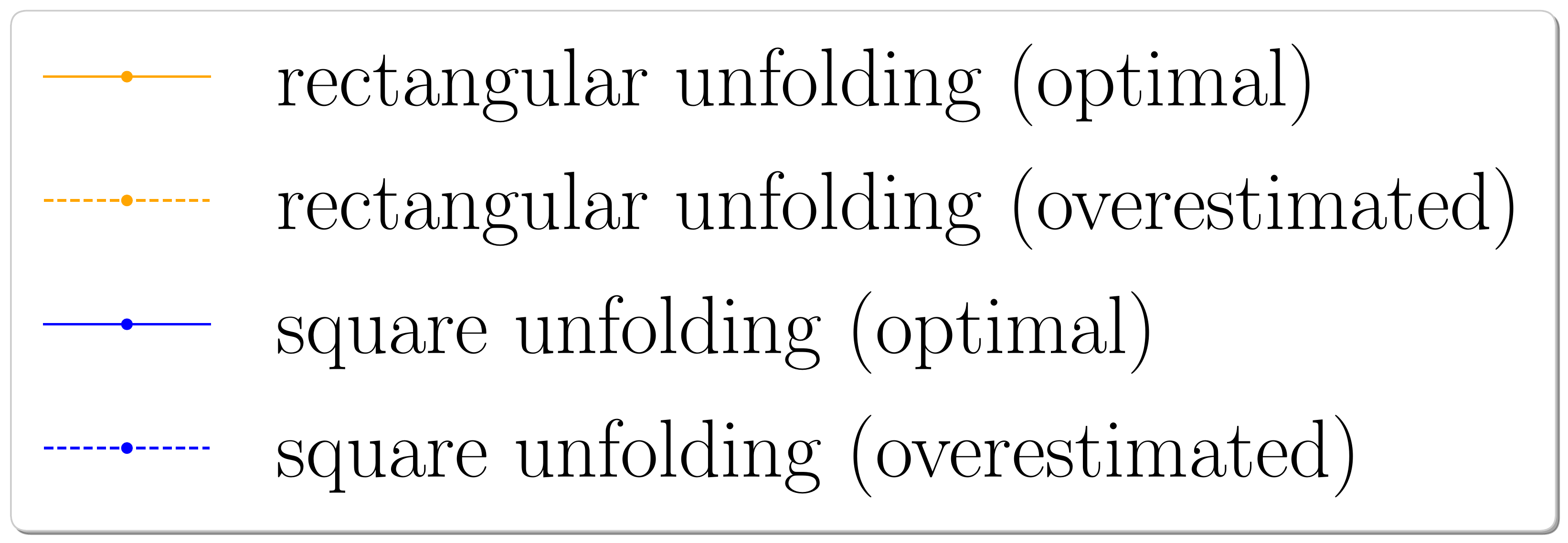}}{}
\end{minipage}
	\caption{Propensity estimation performance on different orders of tensors with size 8.
		The small size makes it possible to evaluate high-order tensors within a reasonable time.
		Figure~\ref{fig:MCAR_relative_error_on_propensity} follows Model A with a $40\%$ observation ratio.
		Figure~\ref{fig:MNAR_relative_error_on_propensity} follows Model B.
	}
	\label{fig:MNAR_propensity}
\end{figure}
Figure~\ref{fig:MNAR_propensity} compares the propensity estimation error on tensors with different orders.
For each $N$, with $I=8$ and $r=2$, we generate an order-$N$ parameter tensor $\T{A}$ in the following way: We first generate $\T{A}^\natural$ by Tucker decomposition $\T{G}^A \times_1 U^A_1 \times \cdots \times_N U^A_N \in \RR^{I \times \cdots \times I}$, in which $\T{G} \in \RR^{r \times \cdots \times r}$ has i.i.d. $\mathcal{N}(0, 10^2)$ entries, and each $U^A_n \in \RR^{I \times r}$ has random orthonormal columns.
Then we generate a noise-corrupted $\T{A}$ by $\T{A}^\natural + (\gamma \fnorm{\T{A}^\natural} / I^{N/2})\upepsilon$, where the noise level $\gamma = 0.1$ and the noise tensor $\upepsilon$ has i.i.d. $\mathcal{N}(0,1)$ entries.
The ``optimal'' hyperparameter setting uses $\tau=\theta$ and $\gamma=\alpha$ in 1-bit matrix completion, and the ``overestimated'' setting, $\tau=2 \theta$ and $\gamma=2 \alpha$.
We can see that:
\begin{enumerate}[label=\arabic*, wide, labelwidth=!, labelindent=0pt]
	\item The square unfolding is always better than the rectangular unfolding in achieving a smaller propensity estimation error.
	\item The propensity estimation error on the square unfolding increases less when the optimization hyperparameters $\tau$ and $\gamma$ increase from their optima, $\theta$ and $\alpha$.
	This suggests that the square unfolding makes the constrained optimization problem more robust to the selection of optimization hyperparameters.
	\item It is most important to use the square unfolding for higher-order tensors: the ratios of relative errors between overestimated and optimal increase from 1.3 to 1.4, while those on the square unfoldings stay at 1.1.
	This aligns with Theorem~\ref{thm:square_unfolding_for_general_matrices} that the square unfolding outperforms in the worst case.
\end{enumerate}

In the second experiment, we complete tensors with MCAR entries.
With $N=4$, $I=100$ and $r=5$, we generate an order-$N$ data tensor $\T{B} \in \RR^{I \times \cdots \times I}$ in the following way: We first generate $\T{B}^\natural$ by $\T{G}^\mathrm{true} \times_1 U_1^\mathrm{true} \times_2 \cdots \times_N U_N^\mathrm{true}$, in which $\T{G}^\mathrm{true} \in \RR^{r \times \cdots \times r}$ has i.i.d. $\mathcal{N}(0,100^2)$ entries, and each $U_n^\mathrm{true} \in \RR^{I \times r}$ has random orthonormal columns.
Then we generate a noise-corrupted $\T{B}$ by $\T{B}^\natural + (\gamma \fnorm{\T{B}^\natural} / I^{N/2})\upepsilon$, where the noise level $\gamma = 0.1$ and the noise tensor $\upepsilon$ has i.i.d. $\mathcal{N}(0,1)$ entries.

\textsc{TenIPS} is competitive for MCAR. 
We compare the relative error of \textsc{TenIPS} with \textsc{SqUnfold}, \textsc{RectUnfold} and \textsc{HOSVD\_w} at different observation ratios in Figure~\ref{fig:MCAR_relative_error_on_ratio}. 
We can see that \textsc{TenIPS} and \textsc{HOSVD\_w} achieve the lowest recovery error on average, and the results of these two methods are nearly identical. 
\begin{figure}
	\centering
	\begin{minipage}[h]{0.26\linewidth}
		\includegraphics[width=\linewidth]{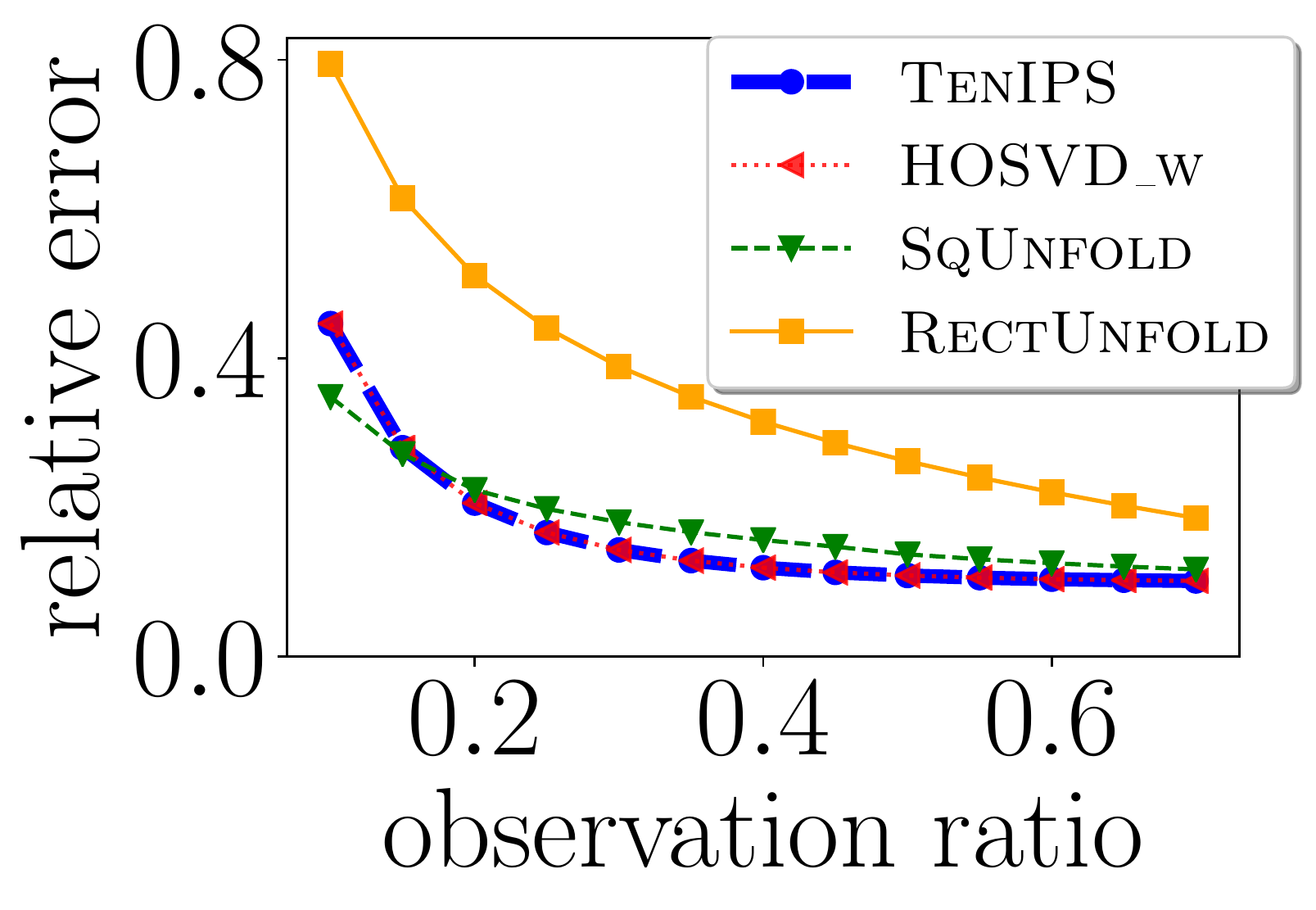}
		\captionof{figure}{Error on MCAR tensors.}
		\label{fig:MCAR_relative_error_on_ratio}
		
		\includegraphics[width=\linewidth]{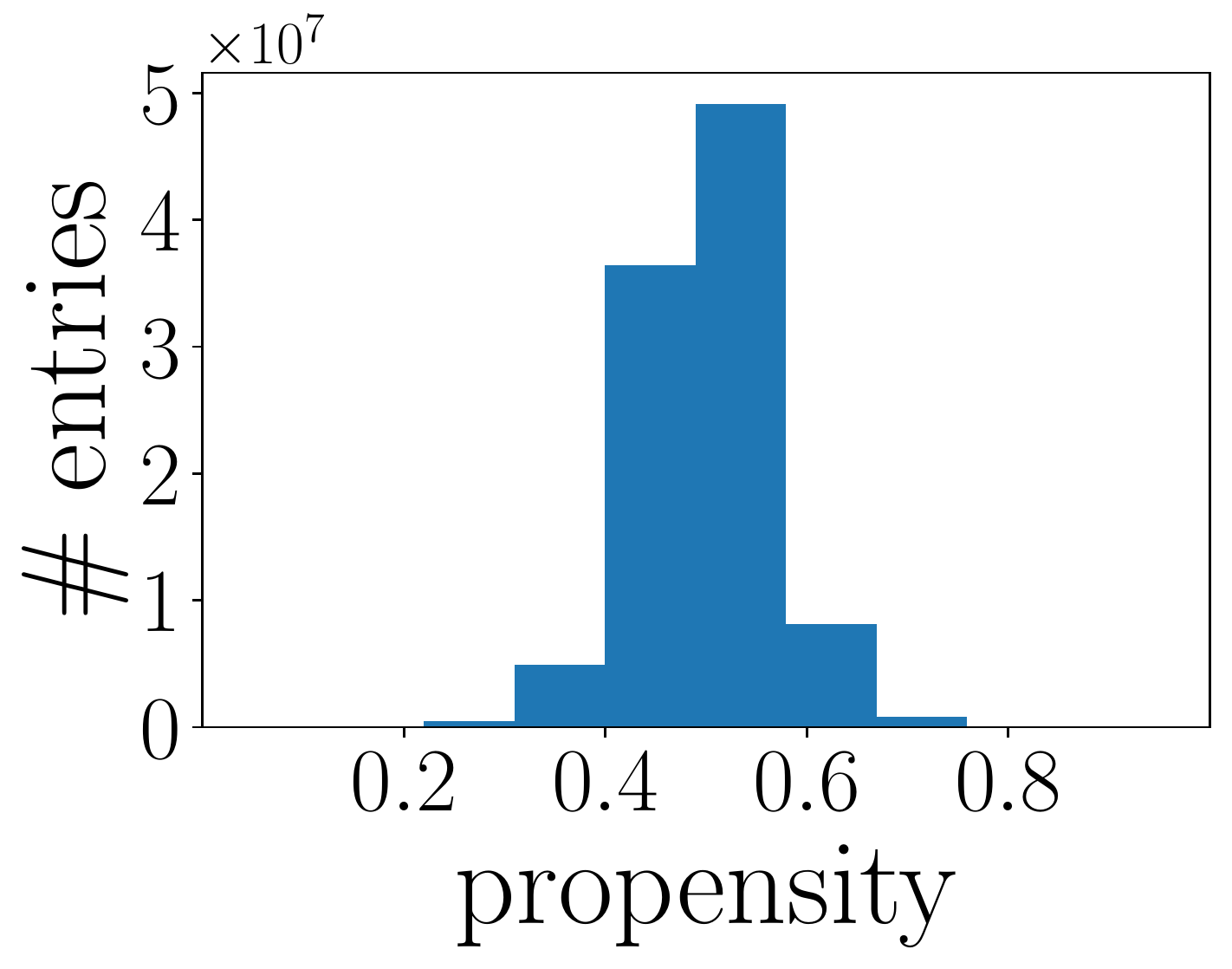}
		\captionof{figure}{Propensity histogram in synthetic MNAR experiments.}
		\label{fig:propensity_histogram}
	\end{minipage}
	\hspace{.02\linewidth}
	\begin{minipage}[h]{0.7\linewidth}
		\centering
		\captionof{table}{Completion performance on the order-4 MNAR synthetic cubical tensor with size 100.
			The ``time'' here is the time taken for the tensor completion step, with true or estimated propensities. $\widehat{\T{P}}_1$ is from running the provable \textsc{ConvexPE} (Algorithm~\ref{alg:propensity_estimation_provable}) at target rank 25 for 84 seconds, and has relative error 0.08 from the true tensor $\T{P}$; $\widehat{\T{P}}_2$ is from running the gradient-descent-based \textsc{NonconvexPE} (Algorithm~\ref{alg:propensity_estimation_alt}) with i.i.d. Uniform$[-1, 1]$ initialization and at step size $5 \times 10^{-6}$ for 81 seconds, and has relative error 0.13.
			The bold number in each column indicates the best.}	
			\begin{tabular}{llllll}
			\toprule
			\multirow{2}{*}{algorithm} & \multirow{2}{*}{time (s)} & \multicolumn{3}{p{3.4cm}}{relative error from $\T{B}$}  \\
			\cmidrule{3-5}
			~ & ~ & \multicolumn{1}{p{1.0cm}}{with $\T{P}$} &  \multicolumn{1}{p{1.2cm}}{with $\widehat{\T{P}}_1$} &  \multicolumn{1}{p{1.2cm}}{with $\widehat{\T{P}}_2$} \\
			\midrule
			\textsc{TenIPS} & 26 & \textbf{0.110} & \textbf{0.110} & \textbf{0.109} \\
			\textsc{HOSVD\_w} & 35 & 0.129 & 0.116 & 0.110 \\
			\textsc{SqUnfold} & 29 & 0.141 & 0.138 & 0.139 \\
			\textsc{RectUnfold} & \textbf{8} & 0.259 & 0.256 & 0.256 \\
			\textsc{LstSq} & >600 & - & - & -\\
			\textsc{SO-HOSVD} & >600 & - & - & -\\
			\bottomrule
		\end{tabular}
		\label{table:MNAR-synthetic-results}		
	\end{minipage}
\end{figure}

In the third experiment, we complete tensors with MNAR entries. 
We use the same $\T{B}$ as the second experiment, and further generate an order-4 parameter tensor $\T{A} \in \RR^{100 \times \cdots \times 100}$ in the same way as $\T{B}$.
$99.97\%$ of the propensities in $\T{P} = \sigma(\T{A})$ lie in the range of $[0.2, 0.8]$, as shown in Figure~\ref{fig:propensity_histogram}.
In Table~\ref{table:MNAR-synthetic-results}, we see:
\begin{enumerate}[label=\arabic*, wide, labelwidth=!, labelindent=0pt]
	\item \textsc{TenIPS} outperforms for MNAR. 
	It has the smallest error among methods that can finish within a reasonable time.
	\item Tensor completion errors using estimated propensities are roughly equal to, and sometimes even smaller than those using true propensities despite a propensity estimation error. 
	\item On the sensitivity to hyperparameters: it is mentioned in the title of Table~\ref{table:MNAR-synthetic-results} that \textsc{ConvexPE} achieves a smaller accuracy within a similar time as \textsc{NonconvexPE}. 
	However, this is because we set $\tau$ and $\gamma$ correctly: $\tau=\theta$ and $\gamma =\alpha$.
	In real cases, $\theta$ and $\alpha$ are unknown, and are hard to infer from surrogate metrics within the optimization process.
	The misestimates of $\theta$ and $\alpha$ may lead to large propensity estimation errors: As an example, \textsc{ConvexPE} with $\tau=100\theta$ and $\gamma =100\alpha$ never achieves a relative error smaller than $0.7$.
	Moreover, the relative error does not always decrease with more PPG iterations, despite the decrease of objective value.
	In every algorithm in Table~\ref{table:MNAR-synthetic-results}, using these estimated propensities yields at least a relative error of $0.7$ for the estimation of data tensor $\T{B}$.
	On the other hand, the initialization and step size in \textsc{NonconvexPE} can be tuned more easily by monitoring the value of function $f$ with the increase of number of iterations. 
	More discussion can be found in Appendix~\ref{sec:algorithm_sensitivity}.
\end{enumerate}

In the fourth experiment, we compare the above methods in both MCAR and MNAR settings when increasing target ranks.
\begin{figure}
	\centering
\centering
	\begin{subfigure}[t]{.26\linewidth}
	\includegraphics[width=\linewidth]{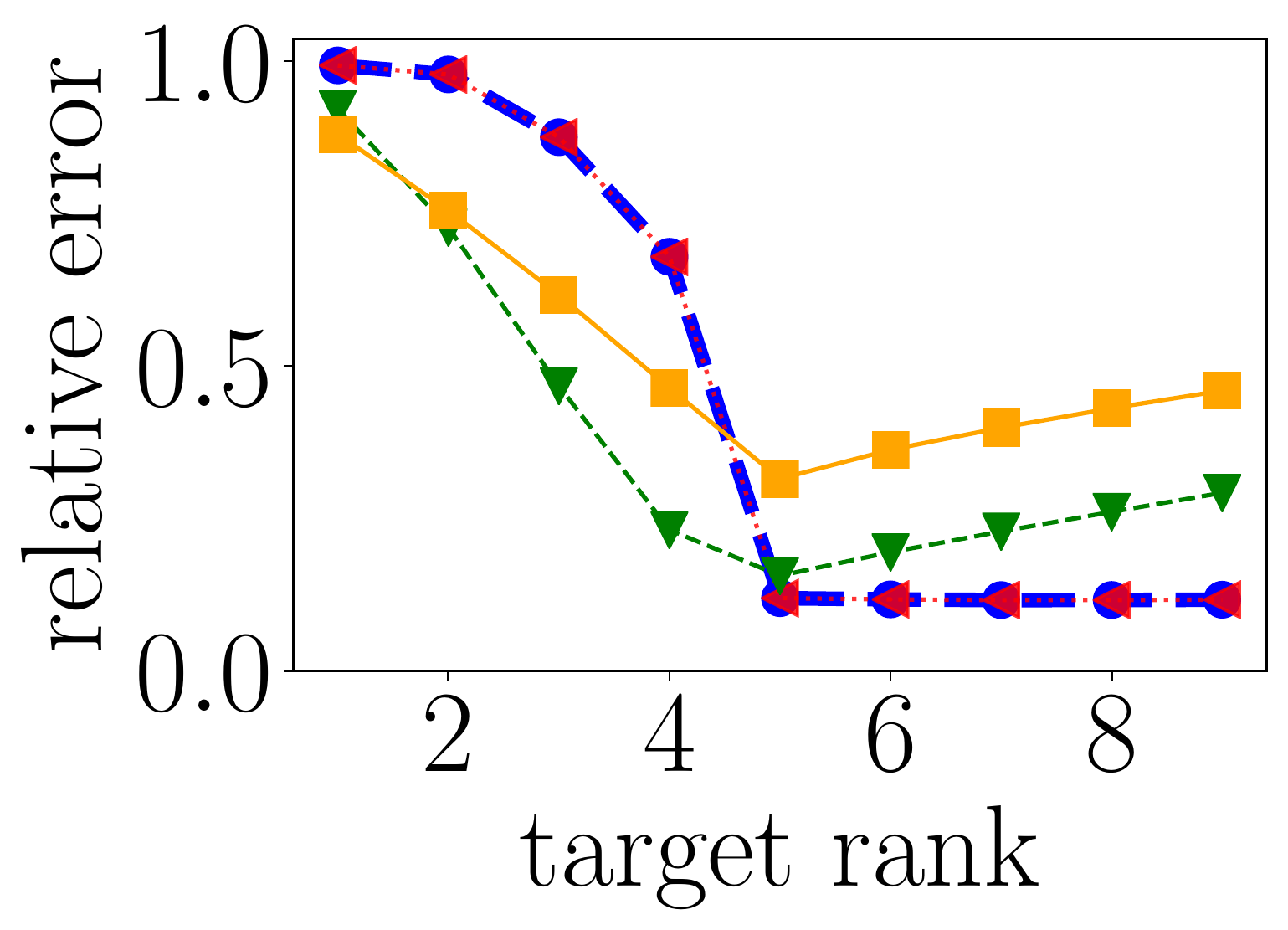}
	\caption{Model A}
\label{fig:MCAR_relative_error_on_target_rank}
\end{subfigure}	
\hspace{.2\linewidth}
\begin{subfigure}[t]{.26\linewidth}
	\includegraphics[width=\linewidth]{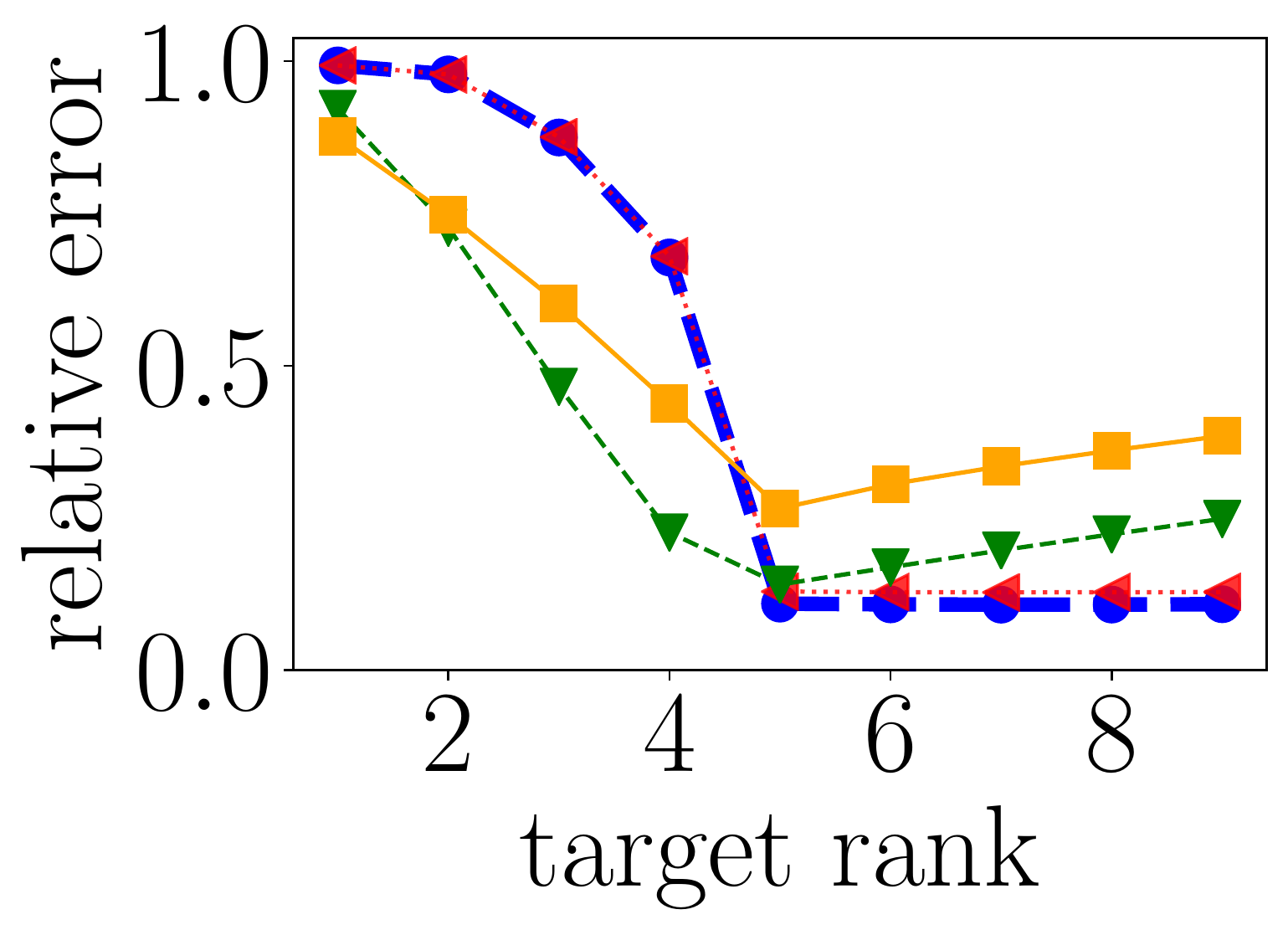}
	\caption{Model B}
\label{fig:MNAR_relative_error_on_target_rank}
\end{subfigure}

\begin{subfigure}[t]{.7\linewidth}
	\includegraphics[width=\linewidth]{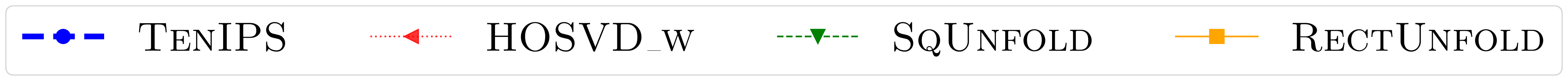}
\end{subfigure}

	\caption{Relative errors at different target ranks on the MNAR data from the third experiment. 
	The true multilinear rank of the data tensor is 5 for each mode.}
	\label{fig:exp_on_varying_ranks}
\end{figure}
In Figure~\ref{fig:exp_on_varying_ranks}, we can see that both \textsc{TenIPS} and \textsc{HOSVD\_w} are more stable at target ranks larger than the true rank, while \textsc{RectUnfold} and \textsc{SqUnfold} achieve smaller errors at smaller ranks. 
This shows that \textsc{TenIPS} and \textsc{HOSVD\_w} are robust to large target ranks, which is the case when $r_n \geq r_n^\mathrm{true}$, for all $n \in [N]$. 

\subsection{Semi-synthetic data}
\label{sec:experiments-semi-synthetic}
We use the video from \cite{malik2018low} and generate synthetic propensities. 
The video was taken by a camera mounted at a fixed position with a person walking by. 
We convert it to grayscale and discard the frames with severe vibration, which yields an order-3 data tensor $\T{B} \in \RR^{2200 \times 1080 \times 1920}$ that takes 102.0GB memory.
To get an MNAR tensor $\T{B}_\mathrm{obs}$, we generate the parameter tensor $\T{A}$ by entrywise transformation $\T{A} = (\T{B} - 128)/64$, which gives propensities in $[0.12, 0.88]$ in $\T{P} = \sigma (\T{A})$.
Finally we subsample $\T{B}$ using propensities $\T{P}$ to get $\T{B}_\mathrm{obs}$.
\begin{figure}[t]
\centering
	\begin{subfigure}[t]{.22\linewidth}
	\includegraphics[width=\linewidth]{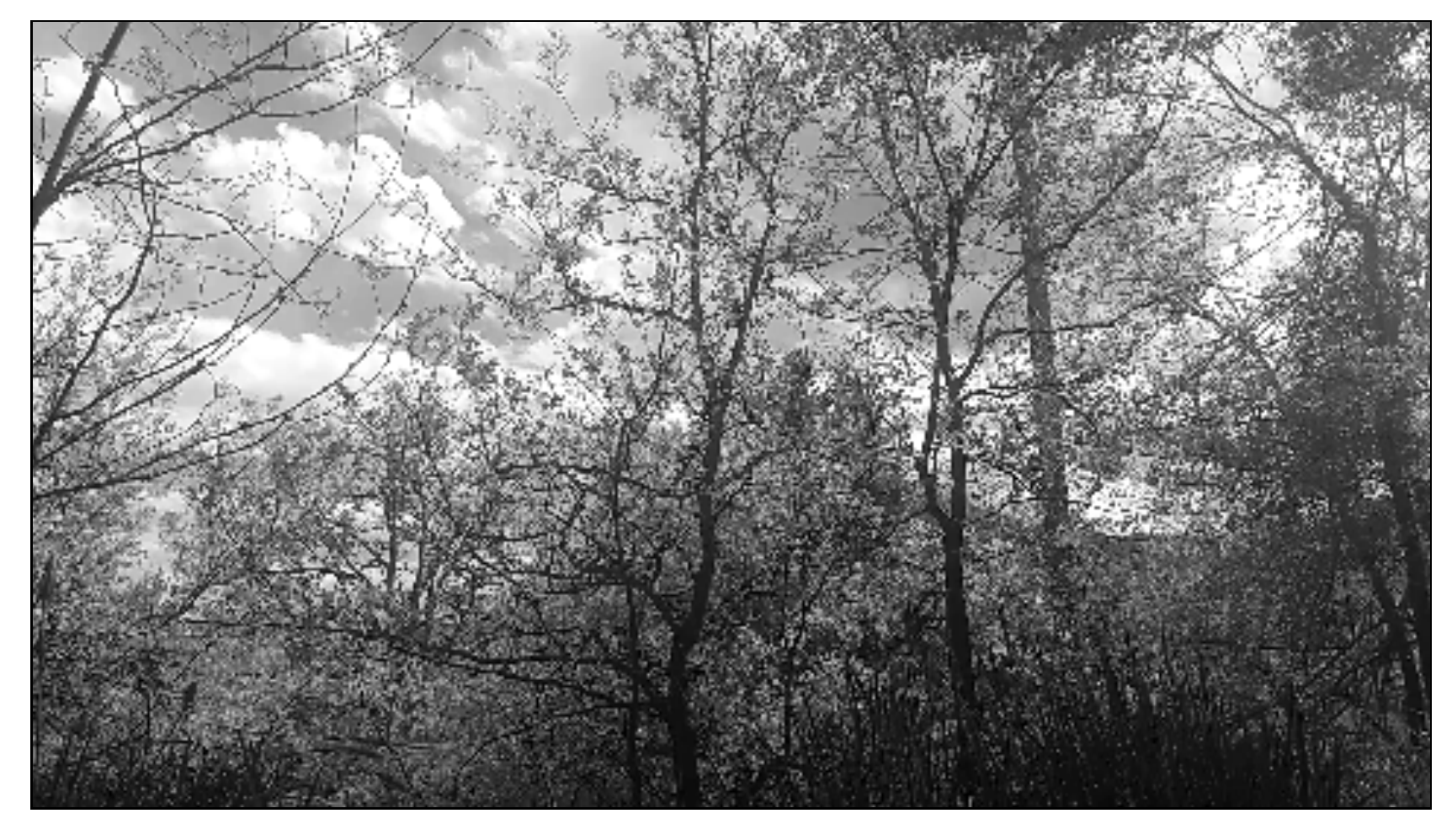}
	\caption{original}
	\label{fig:frame_original}
\end{subfigure}
\hspace{.01\linewidth}
	\begin{subfigure}[t]{.22\linewidth}
	\includegraphics[width=\linewidth]{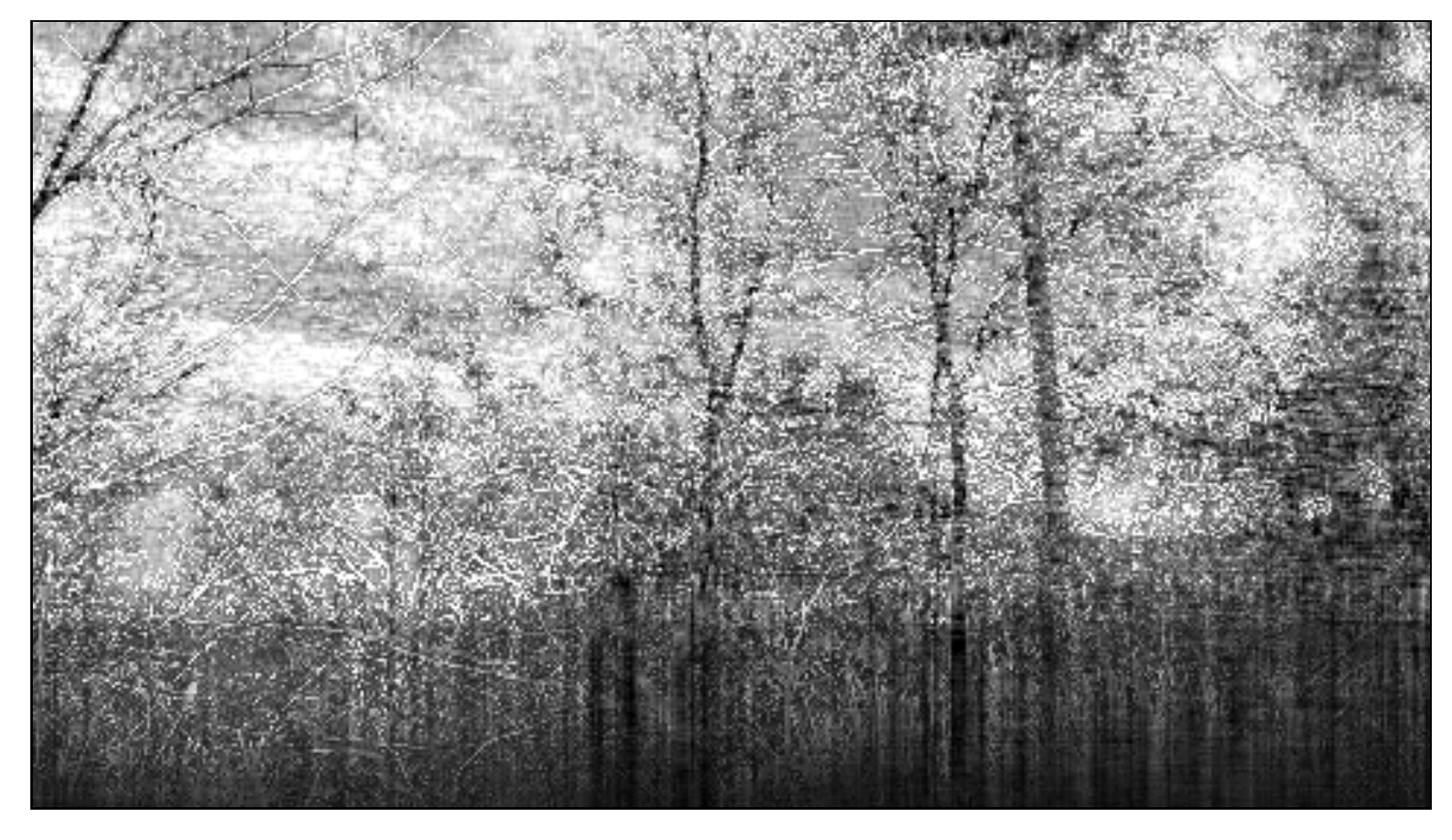}
	\caption{\textsc{TenIPS}, assuming MCAR}
	\label{fig:frame_from_MCAR}
\end{subfigure}
\hspace{.01\linewidth}
	\begin{subfigure}[t]{.22\linewidth}
	\includegraphics[width=\linewidth]{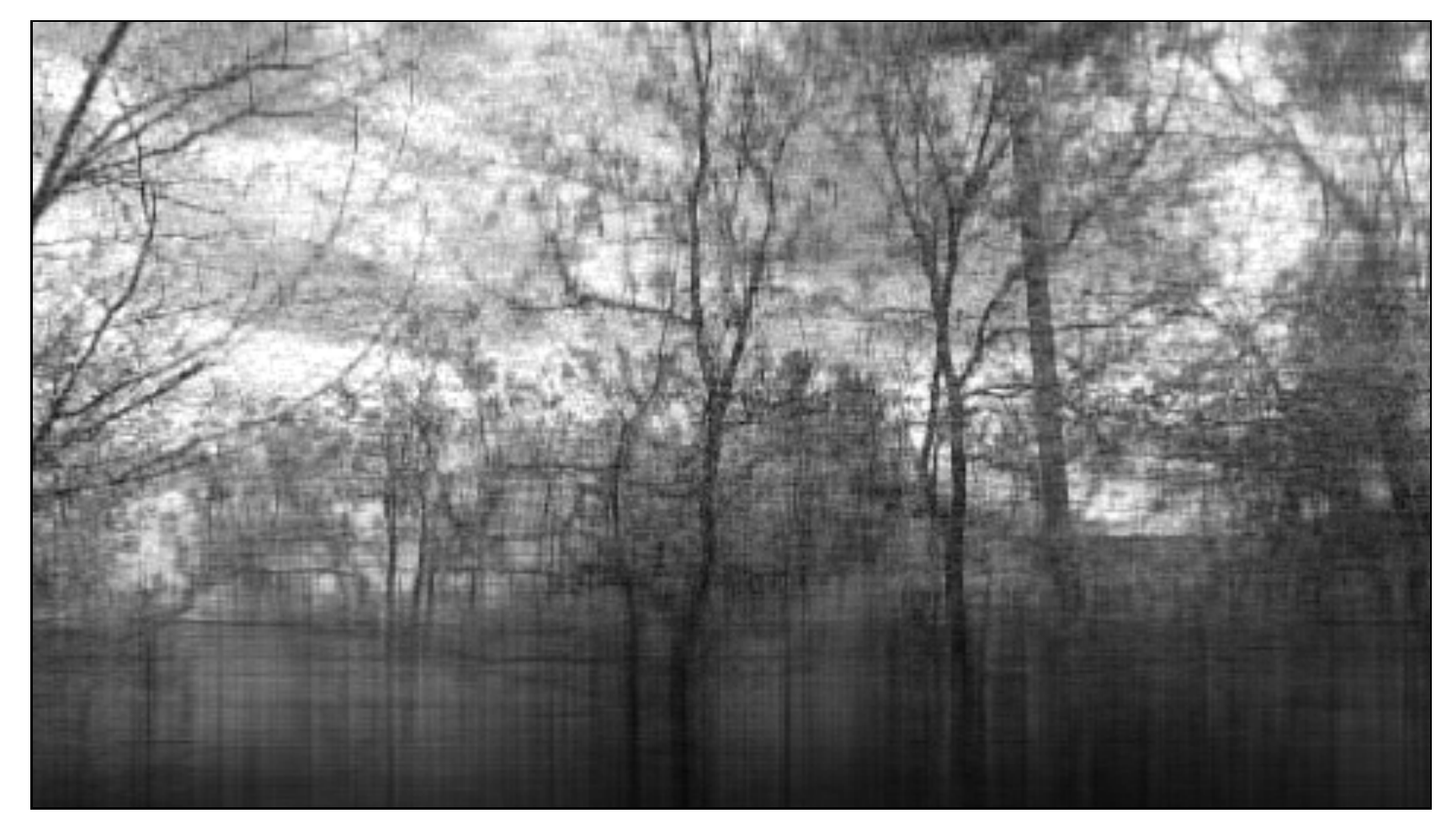}
	\caption{\textsc{TenIPS}, assuming MNAR with true $\T{P}$}
	\label{fig:frame_from_MNAR_true_prop}
\end{subfigure}
\hspace{.01\linewidth}
	\begin{subfigure}[t]{.22\linewidth}
	\includegraphics[width=\linewidth]{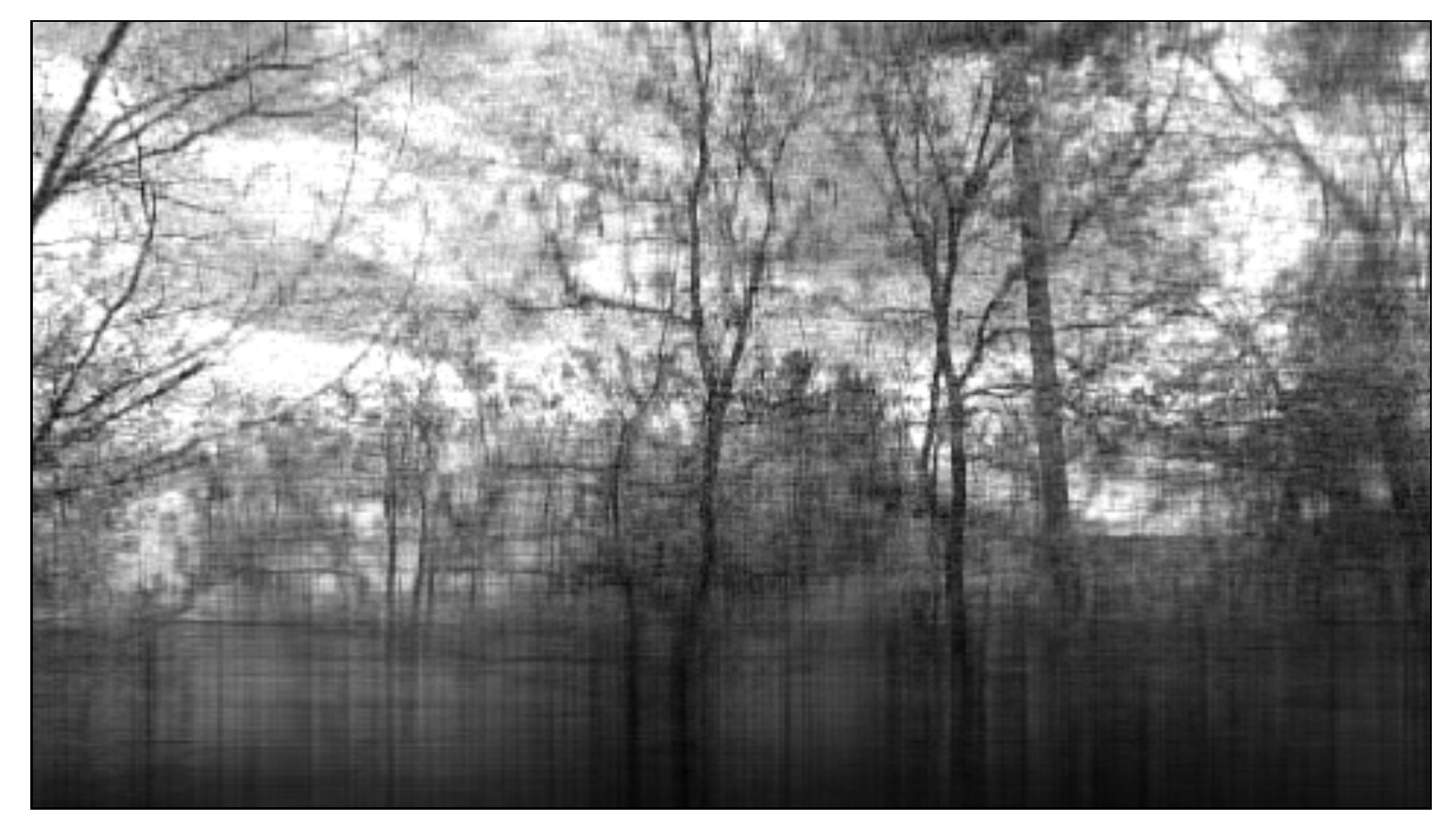}
	\caption{\textsc{TenIPS}, assuming MNAR with estimated $\widehat{\T{P}}$}
	\label{fig:frame_from_MNAR_estimated_prop}
	\end{subfigure}	
	
	\caption{Video recovery visualization on Frame 500 of the \cite{malik2018low} video data.
		The missingness patterns in~\ref{fig:frame_from_MCAR}, \ref{fig:frame_from_MNAR_true_prop} and \ref{fig:frame_from_MNAR_estimated_prop} only refer to our assumption in tensor recovery; the partially observed data tensors we start from are the same and are MNAR.}
	\label{fig:frame}
\end{figure}

We first compare the tensor completion performance of \textsc{TenIPS} with different sources of propensities.
In Figure~\ref{fig:frame}, we visualize the 500-th frame in three \textsc{TenIPS} experiments by fixed-rank approximation with target multilinear rank $(50, 50, 50)$: the original frame without missing pixels~\ref{fig:frame_original}, the frame recovered under MCAR assumption (tensor recovery error 0.42)~\ref{fig:frame_from_MCAR},  the frame recovered by propensities under the MNAR assumption with the true propensity tensor $\T{P}$ (tensor recovery error 0.28)~\ref{fig:frame_from_MNAR_true_prop}, and the frame recovered by propensities under the MNAR assumption with the estimated propensity tensor $\widehat{\T{P}}$ from \textsc{ConvexPE} (propensity estimation error 0.15, tensor recovery error 0.28)~\ref{fig:frame_from_MNAR_estimated_prop}.
We can see that:
\begin{enumerate}[label=\arabic*, wide, labelwidth=!, labelindent=0pt]
	\item With MNAR pixels, the image recovered from the naive MCAR assumption in Figure~\ref{fig:frame_from_MCAR} is more noisy than that from MNAR in Figure~\ref{fig:frame_from_MNAR_true_prop} and \ref{fig:frame_from_MNAR_estimated_prop}, and misses more details.
	\item There is no significant difference between the recovered video frames in \ref{fig:frame_from_MNAR_true_prop} and \ref{fig:frame_from_MNAR_estimated_prop}, in terms of both the frame image itself and the tensor recovery error. 
\end{enumerate}

\begin{table}[t]
	\centering
	\caption{Completion performance on the video task. 
		The memory usages are those in Python 3.
		The bold number in each row indicates the best.
	}	
	\begin{tabular}{llll}
		\toprule
		\multirow{2}{*}{Setting} & \multicolumn{3}{p{6cm}}{relative error (memory in GB)}  \\
		\cmidrule{2-4}
		~ & \multicolumn{1}{p{1.2cm}}{\textsc{TenIPS}} &  \multicolumn{1}{p{1.2cm}}{\textsc{HOSVD\_w}} &  \multicolumn{1}{p{2cm}}{\textsc{RectUnfold}} \\
		\midrule
		\RomanNumeralCaps{1} & 0.28 (0.008) & 0.44 (0.008) & \textbf{0.14} (2.3) \\
		\RomanNumeralCaps{2} & \textbf{0.20} (0.05) & 0.31 (0.05) & 0.25 (0.05) \\
		\bottomrule
	\end{tabular}
	\label{table:video_methods_comparison}
\end{table}

We then compare \textsc{TenIPS} with \textsc{HOSVD\_w} and \textsc{RectUnfold} on this video task;
we omit \textsc{SqUnfold} and \textsc{SO-HOSVD}
because \textsc{SqUnfold} and \textsc{RectUnfold} are equivalent on an order-3 tensor
and \textsc{SO-HOSVD} cannot finish within a reasonable time. 
In Table~\ref{table:video_methods_comparison} Setting~\RomanNumeralCaps{1}, \textsc{TenIPS} and \textsc{HOSVD\_w} have target rank $(50, 50, 50)$, and \textsc{RectUnfold} has target rank $50$.
We can see that \textsc{RectUnfold} has a smaller error in this setting but uses more than $250\times$ memory, because it does not seek low dimensional representations along the two dimensions of the video frame. 

Another advantage of the tensor methods \textsc{TenIPS} and \textsc{HOSVD\_w}
compared to \textsc{RectUnfold}
is that the target rank for different modes is not required to be the same.
For example, if we limit memory usage to 0.05GB (Table~\ref{table:video_methods_comparison}, Setting~\RomanNumeralCaps{2}),
\textsc{TenIPS} and \textsc{HOSVD\_w} can afford a target rank $(5, 500, 500)$
and achieve smaller errors, 
while \textsc{RectUnfold} can only afford a target rank of $1$.

Also, with similar memory consumption in both settings, \textsc{TenIPS} achieves smaller errors than \textsc{HOSVD\_w}. 

\section{Conclusion}
This paper develops a provable two-step approach for MNAR tensor completion with unknown propensities.
The square unfolding allows us to recover propensities with a smaller upper bound, and we then use HOSVD complete MNAR tensor with the estimated propensities. 
This method enjoys theoretical guarantee and fast running time in practice. 

This paper is the first provable method for completing a general MNAR tensor.
There are many avenues for improvement and extensions.
For example, one could explore whether nonconvex matrix completion methods can be generalized to MNAR tensors, explore other observation models, and design provable algorithms that estimate the propensities even faster. 

\subsubsection*{Acknowledgements}
MU, CY, and LD gratefully acknowledge support from
NSF Awards IIS-1943131 and CCF-1740822, 
the ONR Young Investigator Program, 
DARPA Award FA8750-17-2-0101,
the Simons Institute,
Canadian Institutes of Health Research, 
the Alfred P. Sloan Foundation,
and Capital One.
The authors thank Jicong Fan for helpful discussions, and thank several anonymous reviewers for useful comments.

\bibliographystyle{alpha}
\bibliography{scholar}

\newcommand{\etalchar}[1]{$^{#1}$}
\begin{thebibliography}{DPVDBW14}

\bibitem[ATT18]{aidini20181}
Anastasia Aidini, Grigorios Tsagkatakis, and Panagiotis Tsakalides.
\newblock 1-bit tensor completion.
\newblock {\em Electronic Imaging}, 2018(13):261--1, 2018.

\bibitem[AW17]{ashraphijuo2017fundamental}
Morteza Ashraphijuo and Xiaodong Wang.
\newblock Fundamental conditions for low-cp-rank tensor completion.
\newblock {\em The Journal of Machine Learning Research}, 18(1):2116--2145,
  2017.

\bibitem[BM16]{barak2016noisy}
Boaz Barak and Ankur Moitra.
\newblock Noisy tensor completion via the sum-of-squares hierarchy.
\newblock In {\em Conference on Learning Theory}, pages 417--445, 2016.

\bibitem[CC70]{carroll1970analysis}
J~Douglas Carroll and Jih-Jie Chang.
\newblock Analysis of individual differences in multidimensional scaling via an
  {N}-way generalization of ``{Eckart-Young}'' decomposition.
\newblock {\em Psychometrika}, 35(3):283--319, 1970.

\bibitem[CT10]{candes2010power}
Emmanuel~J Cand{\`e}s and Terence Tao.
\newblock The power of convex relaxation: Near-optimal matrix completion.
\newblock {\em IEEE Transactions on Information Theory}, 56(5):2053--2080,
  2010.

\bibitem[CZ13]{cai2013max}
Tony Cai and Wen-Xin Zhou.
\newblock A max-norm constrained minimization approach to 1-bit matrix
  completion.
\newblock {\em The Journal of Machine Learning Research}, 14(1):3619--3647,
  2013.

\bibitem[DK70]{davis1970rotation}
Chandler Davis and William~Morton Kahan.
\newblock The rotation of eigenvectors by a perturbation. iii.
\newblock {\em SIAM Journal on Numerical Analysis}, 7(1):1--46, 1970.

\bibitem[DLDMV00]{de2000multilinear}
Lieven De~Lathauwer, Bart De~Moor, and Joos Vandewalle.
\newblock A multilinear singular value decomposition.
\newblock {\em SIAM journal on Matrix Analysis and Applications},
  21(4):1253--1278, 2000.

\bibitem[DPVDBW14]{davenport20141}
Mark~A Davenport, Yaniv Plan, Ewout Van Den~Berg, and Mary Wootters.
\newblock 1-bit matrix completion.
\newblock {\em Information and Inference: A Journal of the IMA}, 3(3):189--223,
  2014.

\bibitem[GPY18]{ghadermarzy2018learning}
Navid Ghadermarzy, Yaniv Plan, and Ozgur Yilmaz.
\newblock Learning tensors from partial binary measurements.
\newblock {\em IEEE Transactions on Signal Processing}, 67(1):29--40, 2018.

\bibitem[GRY11]{gandy2011tensor}
Silvia Gandy, Benjamin Recht, and Isao Yamada.
\newblock Tensor completion and low-n-rank tensor recovery via convex
  optimization.
\newblock {\em Inverse Problems}, 27(2):025010, 2011.

\bibitem[H{\etalchar{+}}70]{harshman1970foundations}
Richard~A Harshman et~al.
\newblock Foundations of the {PARAFAC} procedure: Models and conditions for an
  ``explanatory'' multimodal factor analysis.
\newblock 1970.

\bibitem[HKD20]{hong2020generalized}
David Hong, Tamara~G Kolda, and Jed~A Duersch.
\newblock Generalized canonical polyadic tensor decomposition.
\newblock {\em SIAM Review}, 62(1):133--163, 2020.

\bibitem[HN20]{huang2020hosvd}
Longxiu Huang and Deanna Needell.
\newblock Hosvd-based algorithm for weighted tensor completion.
\newblock {\em arXiv preprint arXiv:2003.08537}, 2020.

\bibitem[JO14]{jain2014provable}
Prateek Jain and Sewoong Oh.
\newblock Provable tensor factorization with missing data.
\newblock In {\em Advances in Neural Information Processing Systems}, pages
  1431--1439, 2014.

\bibitem[KS13]{krishnamurthy2013low}
Akshay Krishnamurthy and Aarti Singh.
\newblock Low-rank matrix and tensor completion via adaptive sampling.
\newblock In {\em Advances in Neural Information Processing Systems}, pages
  836--844, 2013.

\bibitem[LM20]{liu2020tensor}
Allen Liu and Ankur Moitra.
\newblock Tensor completion made practical.
\newblock {\em arXiv preprint arXiv:2006.03134}, 2020.

\bibitem[MB18]{malik2018low}
Osman~Asif Malik and Stephen Becker.
\newblock Low-rank tucker decomposition of large tensors using tensorsketch.
\newblock In {\em Advances in Neural Information Processing Systems}, pages
  10096--10106, 2018.

\bibitem[MC19]{ma2019missing}
Wei Ma and George~H Chen.
\newblock Missing not at random in matrix completion: The effectiveness of
  estimating missingness probabilities under a low nuclear norm assumption.
\newblock In {\em Advances in Neural Information Processing Systems}, pages
  14871--14880, 2019.

\bibitem[MHWG14]{mu2014square}
Cun Mu, Bo~Huang, John Wright, and Donald Goldfarb.
\newblock Square deal: Lower bounds and improved relaxations for tensor
  recovery.
\newblock In {\em International Conference on Machine Learning}, pages 73--81,
  2014.

\bibitem[NW12]{negahban2012restricted}
Sahand Negahban and Martin~J Wainwright.
\newblock Restricted strong convexity and weighted matrix completion: Optimal
  bounds with noise.
\newblock {\em Journal of Machine Learning Research}, 13(May):1665--1697, 2012.

\bibitem[Ose11]{oseledets2011tensor}
Ivan~V Oseledets.
\newblock Tensor-train decomposition.
\newblock {\em SIAM Journal on Scientific Computing}, 33(5):2295--2317, 2011.

\bibitem[RY17]{ryu2017proximal}
Ernest~K Ryu and Wotao Yin.
\newblock Proximal-proximal-gradient method.
\newblock {\em arXiv preprint arXiv:1708.06908}, 2017.

\bibitem[SGL{\etalchar{+}}19]{sun2019low}
Yiming Sun, Yang Guo, Charlene Luo, Joel Tropp, and Madeleine Udell.
\newblock Low-rank tucker approximation of a tensor from streaming data.
\newblock {\em arXiv preprint arXiv:1904.10951}, 2019.

\bibitem[T{\etalchar{+}}15]{tropp2015introduction}
Joel~A Tropp et~al.
\newblock An introduction to matrix concentration inequalities.
\newblock {\em Foundations and Trends{\textregistered} in Machine Learning},
  8(1-2):1--230, 2015.

\bibitem[THK10]{tomioka2010estimation}
Ryota Tomioka, Kohei Hayashi, and Hisashi Kashima.
\newblock Estimation of low-rank tensors via convex optimization.
\newblock {\em arXiv preprint arXiv:1010.0789}, 2010.

\bibitem[Tuc66]{tucker1966some}
Ledyard~R Tucker.
\newblock Some mathematical notes on three-mode factor analysis.
\newblock {\em Psychometrika}, 31(3):279--311, 1966.

\bibitem[WAA16]{wang2016tensor}
Wenqi Wang, Vaneet Aggarwal, and Shuchin Aeron.
\newblock {Tensor completion by alternating minimization under the tensor train
  (TT) model}.
\newblock {\em arXiv preprint arXiv:1609.05587}, 2016.

\bibitem[Wed72]{wedin1972perturbation}
Per-{\AA}ke Wedin.
\newblock Perturbation bounds in connection with singular value decomposition.
\newblock {\em BIT Numerical Mathematics}, 12(1):99--111, 1972.

\bibitem[XYZ17]{xia2017statistically}
Dong Xia, Ming Yuan, and Cun-Hui Zhang.
\newblock Statistically optimal and computationally efficient low rank tensor
  completion from noisy entries.
\newblock {\em arXiv preprint arXiv:1711.04934}, 2017.

\bibitem[YEG{\etalchar{+}}18]{yokota2018missing}
Tatsuya Yokota, Burak Erem, Seyhmus Guler, Simon~K Warfield, and Hidekata
  Hontani.
\newblock Missing slice recovery for tensors using a low-rank model in embedded
  space.
\newblock In {\em Proceedings of the IEEE Conference on Computer Vision and
  Pattern Recognition}, pages 8251--8259, 2018.

\bibitem[YWS15]{yu2015useful}
Yi~Yu, Tengyao Wang, and Richard~J Samworth.
\newblock A useful variant of the davis--kahan theorem for statisticians.
\newblock {\em Biometrika}, 102(2):315--323, 2015.

\bibitem[YZGC18]{yuan2018high}
Longhao Yuan, Qibin Zhao, Lihua Gui, and Jianting Cao.
\newblock High-dimension tensor completion via gradient-based optimization
  under tensor-train format.
\newblock {\em arXiv preprint arXiv:1804.01983}, 2018.

\bibitem[Z{\etalchar{+}}19]{zhang2019cross}
Anru Zhang et~al.
\newblock Cross: Efficient low-rank tensor completion.
\newblock {\em The Annals of Statistics}, 47(2):936--964, 2019.

\end{thebibliography}

\appendix
\section{Error in tensor completion (\textsc{ConvexPE} and \textsc{TenIPS}): general case}
\label{sec:error-in-tensor-completion-formal}
We first state Theorem~\ref{thm:main-theorem-formal}, the tensor completion error in the most general case.
For brevity, we denote $\T{\Xh}(\T{P})$ and $\bar{\T{X}}(\T{P})$ by $\T{\Xh}$ and $\bar{\T{X}}$, respectively, in which $\T{P}$ is the true propensity tensor.

\begin{theorem}
	\label{thm:main-theorem-formal}
	Consider an order-$N$ tensor $\T{B} \in \mathbb{R}^{I_1 \times \cdots \times I_N}$, and two order-$N$ tensors $\T{P}$ and $\T{A}$ with the same shape as $\T{B}$.
	Each entry $\T{B}_{i_1, \ldots, i_N}$ of $\T{B}$ is observed with probability $\T{P}_{i_1, \ldots, i_N}$ from the corresponding entry of $\T{P}$.
	Assume there exist constants $\psi, \alpha \in (0,\infty)$ such that $\maxnorm{\T{A}} \leq \alpha$, $\maxnorm{\T{B}} = \psi$.
	Denote the spikiness parameter $\alpha_\mathrm{sp} := \psi \sqrt{I_{[N]}} / \fnorm{\T{B}}$.
	Then under the conditions of Lemma~\ref{lem:propensity_error}, with probability at least $\displaystyle 1 - \frac{C_1}{I_\square + I_{\square^C}} - \sum_{n=1}^N [I_n + I_{(-n)}] \exp \Big[- \frac{\epsilon^2 \fnorm{\T{B}}^2 \sigma(- \alpha)/ 2}{I_{(-n)} \psi^2 + \epsilon \psi \fnorm{\T{B}}/3}\Big]$, in which $C_1 > 0$ is a universal constant, the fixed multilinear rank $(r_1, r_2, \cdots, r_N)$ approximation $\widehat{\T{X}}(\widehat{\T{P}})$ computed from \textsc{ConvexPE} and \textsc{TenIPS} (Algorithms~\ref{alg:propensity_estimation_provable} and \ref{alg:tensor_completion}) with thresholds $\tau \geq \theta$ and $\gamma \geq \alpha$ satisfies
	\begin{equation}
		\begin{aligned}
			\frac{\fnorm{\widehat{\T{X}}(\widehat{\T{P}}) - \T{B}}^2}{\fnorm{\T{B}}^2} \leq & \min_{n \in [N]} \Bigg\{r_n \cdot \Big[\frac{\fnorm{\bar{\T{X}}(\T{\Ph}) - \bar{\T{X}}}}{\fnorm{\T{B}}} + \epsilon \Big]^2 \Bigg\} \\
			& + \sum_{n=1}^N \frac{12 r_n \sigma_1 (\T{B}^{(n)})^2}{\fnorm{\T{B}}^2}  \cdot \Bigg\{\frac{[2 \sigma_1 (\T{B}^{(n)}) + \fnorm{\bar{\T{X}}(\T{\Ph}) - \bar{\T{X}}} + \epsilon \fnorm{\T{B}}]^2}{[\sigma_{r_n} (\T{B}^{(n)}) + \sigma_{r_n+1} (\T{B}^{(n)})]^2} \cdot \frac{[\fnorm{\bar{\T{X}}(\T{\Ph}) - \bar{\T{X}}} + \epsilon \fnorm{\T{B}}]^2}{[\sigma_{r_n} (\T{B}^{(n)}) - \sigma_{r_n+1} (\T{B}^{(n)})]^2} \Bigg\}\\
			& + \frac{1}{\fnorm{\T{B}}^2} \sum_{n=1}^{N} (\tau_{r_n}^{(n)})^2, \\
		\end{aligned}
		\label{eq:main-thm-bound}
	\end{equation}
	in which:
	\begin{enumerate}
		\item $(\tau_{r_n}^{(n)})^2 := \sum_{i=r_n + 1}^{I_n} \sigma_i^2(\T{B}^{(n)})$ is the $r_n$-th tail energy for $\T{B}^{(n)}$,
		\item from Lemma~\ref{lem:propensity_error}, with $L_\gamma = \sup_{x\in[-\gamma,\gamma]} \frac{|\sigma'(x)|}{\sigma(x)(1-\sigma(x))}$, and with probability at least $1 - \frac{C_1}{I_\square + I_{\square^C}}$,
		\begin{equation}
			\label{eq:fnorm-error-on-exact-prop}
			\fnorm{\bar{\T{X}}(\T{\Ph}) - \bar{\T{X}}} \leq \frac{\alpha_\mathrm{sp} \fnorm{\T{B}}}{\sigma(-\gamma) \sigma(-\alpha)} \sqrt{4 e L_\gamma \tau \Big(\frac{1}{\sqrt{I_\square}}+\frac{1}{\sqrt{I_{\square^C}}}\Big)}.
		\end{equation}
	\end{enumerate}
\end{theorem}

On the right-hand side of Equation~\ref{eq:main-thm-bound}, the first term comes from the error between $\bar{\T{X}}(\T{P})$ and $\T{B}$ when projected onto the truncated column singular spaces in each mode $n \in [N]$; the second and third terms come from the projection error of $\T{B}$ onto the above spaces. 

Then, Theorem~\ref{thm:main-theorem-informal} is a corollary of the above Theorem~\ref{thm:main-theorem-formal} in the special case that the tensor is cubical and every unfolding has the same rank.

\section{Proof for Theorem~\ref{thm:main-theorem-informal} and~\ref{thm:main-theorem-formal}}
\label{sec:proof-for-main-theorem}
\subsection{Proof for Theorem~\ref{thm:main-theorem-formal}, the general case}
We first show the proof for Theorem~\ref{thm:main-theorem-formal}, the general case. 
This is the full version of the proof sketch in Section~\ref{sec:error-in-tensor-completion-informal} of the main paper.
We start with Lemma~\ref{lem:propensity-error-on-x-bar} on how the error in propensity estimates propagates to the error in the inverse propensity estimator $\bar{\T{X}}(\widehat{\T{P}})$, then bound the error between $\widehat{\T{X}}(\widehat{\T{P}})$ and $\T{B}$.
\begin{lemma}
	\label{lem:propensity-error-on-x-bar}
	Instate the conditions of Lemma 2 and further suppose $\maxnorm{\T{B}} = \psi$.
	Then with probability at least $1 - \frac{C_1}{I_S + I_{S^C}}$, in which $C_1 > 0$ is a universal constant,
	\begin{equation}
		\label{eq:x_bar_error_from_propensity_revisited}
		\fnorm{\bar{\T{X}}(\T{\Ph}) - \bar{\T{X}}}^2 \leq \frac{4 e L_\gamma \tau \psi^2}{\sigma(-\gamma)^2 \sigma(-\alpha)^2} \Big(\frac{1}{\sqrt{I_S}}+\frac{1}{\sqrt{I_{S^C}}}\Big) I_{[N]}.
	\end{equation}
\end{lemma}
\begin{proof}
	Under the above conditions,
	\begin{equation}
		\begin{aligned}
			\fnorm{\bar{\T{X}}(\T{\Ph}) - \bar{\T{X}}}^2 & = \sum_{({i_1, i_2, \dots, i_N}) \in \Omega} \T{B}_{i_1 i_2 \cdots i_N}^2 (\frac{1}{\T{P}_{i_1 i_2 \cdots i_N}} - \frac{1}{\T{\Ph}_{i_1 i_2 \cdots i_N}})^2 \\
			& \leq  \psi^2 \sum_{({i_1, i_2, \dots, i_N}) \in \Omega} \Big(\frac{\T{P}_{i_1 i_2 \cdots i_N} - \T{\Ph}_{i_1 i_2 \cdots i_N}}{\T{P}_{i_1 i_2 \cdots i_N} \T{\Ph}_{i_1 i_2 \cdots i_N}} \Big)^2 \\
			& \leq  \frac{\psi^2}{\sigma(-\gamma)^2 \sigma(-\alpha)^2} \sum_{({i_1, i_2, \dots, i_N}) \in \Omega} \Big(\T{P}_{i_1 i_2 \cdots i_N} - \T{\Ph}_{i_1 i_2 \cdots i_N} \Big)^2 \\
			& \leq \frac{4 e L_\gamma \tau \psi^2}{\sigma(-\gamma)^2 \sigma(-\alpha)^2} \Big(\frac{1}{\sqrt{I_S}}+\frac{1}{\sqrt{I_{S^C}}}\Big) I_{[N]}.
		\end{aligned}
		\nonumber
	\end{equation}
	The second inequality comes from $ \T{\Ph}_{i_1 i_2 \cdots i_N} \geq \sigma(-\gamma)$ and $ \T{P}_{i_1 i_2 \cdots i_N} \geq \sigma(-\alpha)$; the last inequality follows Lemma~\ref{lem:propensity_error}.
\end{proof}

We then state two lemmas that we will apply to tensor unfoldings.
Lemma~\ref{lem:mtx-bernstein} is the matrix Bernstein inequality.
Lemma~\ref{lem:davis-kahan-variant} is a variant of the Davis-Kahan sin($\Theta$) Theorem \cite{davis1970rotation}.

\begin{lemma}[Matrix Bernstein for real matrices {\cite[Theorem 1.6.2]{tropp2015introduction}}]
	\label{lem:mtx-bernstein}	
	Let $S_1, \ldots, S_k$ be independent, centered random matrices with common dimension $m \times n$,
	and assume that each one is uniformly bounded
	$$
	\Expect S_i = 0
	\quad\text{and}\quad
	\norm{S_i} \leq L
	\quad\text{for each $i = 1, \ldots, k$.}
	$$
	Introduce the sum
	\begin{equation} 
		Z = \sum_{i=1}^k S_i,
		\nonumber
	\end{equation}
	and let $v(Z)$ denote the matrix variance statistic of the sum:
	\begin{equation}
		\begin{aligned}
			v(Z) &= \max\big\{ \norm{ \smash{\Expect (Z Z^\top)} }, \ 
			\norm{ \smash{\Expect (Z^\top Z) } } \big\} \\
			&= \max\left\{ \norm{ \sum_{i=1}^k \Expect \big(S_i S_i^\top \big) }, \
			\norm{ \sum_{i=1}^k \Expect \big(S_i^\top S_i \big) } \right\}.
		\end{aligned}
		\nonumber
	\end{equation}
	Then
	\begin{equation}
		\Prob{ \norm{ Z } \geq t }
		\leq (m + n) \cdot \exp \left( \frac{-t^2/2}{v(Z) + Lt/3} \right)
		\quad\text{for all $t \geq 0$.}
		\nonumber
	\end{equation}
\end{lemma}

\begin{lemma}[Variant of the Davis-Kahan sin($\Theta$) Theorem {\cite{wedin1972perturbation}, \cite[Theorem 4]{yu2015useful}}]
	\label{lem:davis-kahan-variant}
	Let $A, \Ah \in \mathbb{R}^{p \times q}$ have singular values $\sigma_1 \geq \ldots \geq \sigma_{\min(p,q)}$ and $\widehat{\sigma}_1 \geq \ldots \geq \widehat{\sigma}_{\min(p,q)}$ respectively, and have singular vectors $\{u_i\}_{i=1}^n$, $\{v_i\}_{i=1}^n$ and $\{\uh_i\}_{i=1}^n$, $\{\vh_i\}_{i=1}^n$, respectively.
	Let $V = (v_1, \cdots, v_r) \in \RR^{n \times r}$, $\Vh = (\vh_1, \cdots, \vh_r) \in \RR^{n \times r}$, $V_\perp = (v_{r+1}, \cdots, v_{n}) \in \RR^{n \times (n - r)}$ and $\Vh_\perp = (\vh_{r+1}, \cdots, \vh_{n}) \in \RR^{n \times (n - r)}$.
	Assume that $\sigma_r^2 - \sigma_{r+1}^2 > 0$, then
	\begin{equation}
		\fnorm{\Vh_\perp^\top V} = \fnorm{V_\perp^\top \Vh} = \fnorm{\Vh \Vh^\top - V V^\top} \leq \frac{2(2\sigma_1 + \norm{\Ah - A}) \min(r^{1/2} \norm{\Ah - A}, \fnorm{\Ah - A})}{\sigma_r^2 - \sigma_{r+1}^2}.
		\nonumber
	\end{equation}
	
	Identical bounds also hold if $V$ and $\Vh$ are replaced with the matrices of left singular vectors $U$ and $\Uh$, where $U = (u_r,u_{r+1},\ldots,u_s) \in \mathbb{R}^{p \times d}$ and $\Uh = (\uh_r,\uh_{r+1},\ldots,\uh_s) \in \mathbb{R}^{p \times d}$ have orthonormal columns satisfying $A^\top u_j = \sigma_j v_j$ and $\Ah^\top \hat{u}_j = \widehat{\sigma}_j \vh_j$ for $j= r,r+1,\ldots,s$.
	
\end{lemma}

\paragraph{Upper bound on $\norm{\bar{\T{X}}^{(n)}(\T{\Ph}) - \T{B}^{(n)}}$:}
\label{sec:spectral-norm-upper-bound}
We decompose it into the error between $\bar{\T{X}}^{(n)}(\T{\Ph})$ and $\bar{\T{X}}^{(n)}(\T{P})$, and the error between $\bar{\T{X}}^{(n)}(\T{P})$ and $\T{B}$, and independently bound these two terms:
\begin{equation}
	\label{eq:general_sampling_estimator_error_revisited}
	\begin{aligned}
		\norm{\bar{\T{X}}^{(n)}(\T{\Ph}) - \T{B}^{(n)}} & \leq \norm{\bar{\T{X}}^{(n)}(\T{\Ph}) - \bar{\T{X}}^{(n)}} + \norm{\bar{\T{X}}^{(n)} - \T{B}^{(n)}} \\
		& \leq \fnorm{\bar{\T{X}}^{(n)}(\T{\Ph}) - \bar{\T{X}}^{(n)}} + \norm{\bar{\T{X}}^{(n)} - \T{B}^{(n)}}.\\
	\end{aligned}
\end{equation}

The first RHS term is bounded by Lemma~\ref{lem:propensity-error-on-x-bar}, the error given by propensity estimation.
Note that we can get a tighter bound if we can directly bound $\norm{\bar{\T{X}}^{(n)}(\T{\Ph}) - \bar{\T{X}}^{(n)}}$. The second RHS term can be bounded by Lemma~\ref{lem:mtx-bernstein}, the matrix Bernstein inequality, as below.

For each $(i_1, \dots, i_N)$, define the random variable
\begin{equation}
	\T{S}_{i_1 i_2 \cdots i_N} := \left\{
	\begin{aligned}
		& \Big(\frac{1}{\T{P}_{i_1 i_2 \cdots i_N}} - 1 \Big) \T{B} \odot \T{E}(i_1, i_2, \dots, i_N), \quad && \text{with probability $\T{P}_{i_1 i_2 \cdots i_N}$} \\
		& - \T{B} \odot \T{E}(i_1, i_2, \dots, i_N), \quad && \text{with probability $1-\T{P}_{i_1 i_2 \cdots i_N}$.}
	\end{aligned}
	\right.
	\nonumber
\end{equation}
With the assumptions in Theorem~\ref{thm:main-theorem-formal}, $\Expect \T{S}_{i_1 i_2 \cdots i_N} = 0$ and $\norm{\T{S}_{i_1 i_2 \cdots i_N}^{(n)}} \leq \frac{\psi}{\sigma (- \alpha)}$. 
Also, the per-mode second moment is bounded as
\begin{equation}
	\begin{aligned}
		v_n(\T{X}) & = \max\{\norm{\sum_{i_1=1}^{I_1} \sum_{i_2=1}^{I_2} \cdots \sum_{i_n=1}^{I_n} \Expect[\T{S}_{i_1 i_2 \cdots i_N}^{(n)} (\T{S}_{i_1 i_2 \cdots i_N}^{(n)})^\top] }, \norm{\sum_{i_1=1}^{I_1} \sum_{i_2=1}^{I_2} \cdots \sum_{i_n=1}^{I_n} \Expect[(\T{S}_{i_1 i_2 \cdots i_N}^{(n)})^\top \T{S}_{i_1 i_2 \cdots i_N}^{(n)}]} \} \\
		& \leq \frac{\psi^2 \cdot I_{(-n)}}{\sigma (- \alpha)}.
	\end{aligned}
	\nonumber
\end{equation}
With probability at least $1 - [I_n + I_{(-n)}] \exp \Big[- \frac{\epsilon^2 \fnorm{\T{B}}^2 \sigma(- \alpha)/ 2}{I_{(-n)} \psi^2 + \epsilon \psi \fnorm{\T{B}}/3}\Big]$, the sum of random variables is bounded as $\displaystyle \norm{ \sum_{i_1=1}^{I_1} \sum_{i_2=1}^{I_2} \cdots \sum_{i_N=1}^{I_N} \T{S}_{i_1 i_2 \cdots i_N} } \leq \epsilon \fnorm{\T{B}}$.
Notice the difference between the propensity-reweighted observed tensor $\bar{\T{X}}(\T{P})$ and the true tensor $\T{B}$ 
\begin{equation}
	\begin{aligned}
	\bar{\T{X}} (\T{P}) - \T{B} & = \sum_{(i_1, i_2, \dots, i_N) \in \Omega} \frac{1}{\T{P}_{i_1 i_2 \cdots i_N}} \T{B}_\text{obs} \odot \T{E}(i_1, i_2, \dots, i_N) - \T{B} \\
	& = \sum_{(i_1, i_2, \dots, i_N) \in \Omega} \Big(\frac{1}{\T{P}_{i_1 i_2 \cdots i_N}} - 1 \Big) \T{B} \odot \T{E}(i_1, i_2, \dots, i_N) - \sum_{(i_1, i_2, \dots, i_N) \notin \Omega}  \T{B} \odot \T{E}(i_1, i_2, \dots, i_N)\\
	\end{aligned}
	\nonumber
\end{equation}
is an instance of $\displaystyle \sum_{i_1=1}^{I_1} \sum_{i_2=1}^{I_2} \cdots \sum_{i_N=1}^{I_N} \T{S}_{i_1 i_2 \cdots i_N}$ over the randomness of entry-wise observation, hence we can use the matrix Bernstein inequality (Lemma~\ref{lem:mtx-bernstein}) to bound $\norm{\bar{\T{X}} (\T{P}) - \T{B}}$.
Together with Equations~\ref{eq:x_bar_error_from_propensity_revisited} and~\ref{eq:general_sampling_estimator_error_revisited}, we get the upper bound on $\norm{\bar{\T{X}}^{(n)}(\T{\Ph}) - \T{B}^{(n)}}$.

\paragraph{How $\norm{\bar{\T{X}}^{(n)}(\T{\Ph}) - \T{B}^{(n)}}$ propagates into the final error in \textsc{TenIPS} (Algorithm~\ref{alg:tensor_completion}):}
In \textsc{TenIPS},
\begin{equation}
	\T{\Xh} (\T{\Ph})= \underbrace{\left[\bar{\T{X}}(\T{\Ph}) \times Q_1^\top \times_2 \cdots \times_N Q_N^\top\right]}_{\T{W}(\T{\Ph})} \times_1 Q_1 \times_2 \cdots \times_N Q_N = \bar{\T{X}}(\T{\Ph}) \times_1 Q_1Q_1^\top \times_2 \cdots \times_N Q_N Q_N^\top.
	\nonumber
\end{equation}
This projects each unfolding of $\bar{\T{X}} (\T{\Ph})$ onto the space of its top left singular vectors. 
Thus by adding and subtracting $ \T{B} \times_1 Q_1Q_1^\top \times_2 \cdots \times_N Q_N Q_N^\top$ within the Frobenius norm, we decompose the error as
\begin{equation}
	\begin{aligned}
		\fnorm{\T{\Xh}(\T{\Ph}) - \T{B}}^2 & =  \fnorm{\bar{\T{X}}(\T{\Ph}) \times_1 Q_1Q_1^\top \times_2 \cdots \times_N Q_N Q_N^\top - \T{B}}^2 \\
		& = \fnorm{\bar{\T{X}}(\T{\Ph}) \times_1 Q_1Q_1^\top \times_2 \cdots \times_N Q_N Q_N^\top - \T{B} \times_1 Q_1Q_1^\top \times_2 \cdots \times_N Q_N Q_N^\top \\
			& \quad + \T{B} \times_1 Q_1Q_1^\top \times_2 \cdots \times_N Q_N Q_N^\top - \T{B}}^2 \\
		& = \underbrace{\fnorm{(\bar{\T{X}}(\T{\Ph})-\T{B}) \times_1 Q_1Q_1^\top \times_2 \cdots \times_N Q_N Q_N^\top}^2}_{\textcircled{1}}\\
		& +  \underbrace{\fnorm{\T{B} \times_1 Q_1Q_1^\top \times_2 \cdots \times_N Q_N Q_N^\top - \T{B}}^2}_{\textcircled{2}}\\
		& + \underbrace{2 \langle (\bar{\T{X}}(\T{\Ph})-\T{B}) \times_1 Q_1Q_1^\top \times_2 \cdots \times_N Q_N Q_N^\top, \T{B} \times_1 Q_1Q_1^\top \times_2 \cdots \times_N Q_N Q_N^\top - \T{B} \rangle}_{\textcircled{3}}. \\
	\end{aligned}
	\nonumber
\end{equation}

First, we show that the cross term $\textcircled{3}$ is zero, since it is the product of two terms in mutually orthogonal subspaces. 
For each $n \in [N]$,
\begin{equation}
	[(\bar{\T{X}}(\T{\Ph})-\T{B}) \times_1 Q_1 Q_1^\top \times_2 \cdots \times_N Q_N Q_N^\top]^{(n)} = Q_{n} \T{C}_n^{(n)},
	\nonumber
\end{equation}
where $\T{C}_n^{(n)}$ is the mode-$n$ unfolding of the tensor $\T{C}_n$, defined as
\begin{equation}
	\T{C}_n = [(\bar{\T{X}}(\T{\Ph})-\T{B})\times_1 Q_1^\top \cdots \times_N Q_N^\top]
	\times_1 Q_1 \dots \times_{n-1} Q_{n-1} \times_{n+1}  Q_{n+1} \dots \times_N Q_{N}.
	\nonumber
\end{equation}

Thus we have
\begin{equation}
	\begin{aligned}
		\textcircled{3} & =2 \sum_{n=1}^{N} \langle \T{Y}_{n} - \T{Y}_{n-1}, (\bar{\T{X}}(\T{\Ph})-\T{B}) \times_1 Q_1Q_1^\top \times_2 \cdots \times_N Q_N Q_N^\top\rangle \\
		& = 2\langle (Q_{n} Q^\top_{n} - I) \T{Y}_{n-1}^{(n)}, Q_{n} \T{C}_n^{(n)}\rangle \\
		& = 2 \mathrm{tr}(\T{Y}_{n-1}^{(n)} (Q_{n} Q^\top_{n} - I)Q_{n} \T{C}_n^{(n)} )=0 .
	\end{aligned}
	\nonumber
\end{equation}

Next, for Terms $\textcircled{1}$ and $\textcircled{2}$, we introduce more notation before we analyze the error.
Define $\T{Y}_0 = \T{B}$, and for each $n \in [N]$ let
\begin{equation}
	\label{eq:definition_Y_n}
	\T{Y}_n = \T{B} \times_1 Q_1 Q_1^\top \times_2 \cdots \times_n Q_n Q_n^\top.
	\nonumber
\end{equation}

Thus $\T{B} \times_1 Q_1Q_1^\top \times_2 \cdots \times_N Q_N Q_N^\top - \T{B} = \T{Y}_N - \T{Y}_0 = \sum_{n=1}^{N} (\T{Y}_n - \T{Y}_{n-1}).$
Each $n \in [N]$ in the sum satisfies
\begin{equation}
	\T{Y}_{n} - \T{Y}_{n-1} = \T{Y}_{n-1} \times_{n} (Q_{n} Q^\top_{n} - I).
	\nonumber
\end{equation}
This allows us to analyze each mode individually.

For Term $\textcircled{1}$, for any $n \in [N]$, we have
\begin{equation}
	\begin{aligned}
		\textcircled{1} & \leq \min_{n \in [N]} \Big\{\fnorm{Q_n Q_n^\top (\bar{\T{X}}(\T{\Ph})^{(n)} - \T{B}^{(n)})}^2 \Big\}\\
		& \leq \min_{n \in [N]} \Big\{r_n \cdot \norm{\bar{\T{X}}(\T{\Ph})^{(n)} - \T{B}^{(n)}}^2 \Big\},\\
	\end{aligned}
	\nonumber
\end{equation}
the RHS of which can be bounded from Section~\ref{sec:spectral-norm-upper-bound}.

As for Term $\textcircled{2}$, it can be bounded using a technique similar to \cite[Lemma B.1]{sun2019low}.
For each $n \in [N]$,
\begin{equation}
	\begin{aligned}
		\fnorm{\T{Y}_{n} - \T{Y}_{n-1}}^2 &= \fnorm{\T{B}\times_n (I - Q_{n} Q_{n}^\top)\times_{1} Q_{1}Q_{1}^\top \dots \times_n Q_{n-1} Q_{n-1}^\top}^2 \\
		&\leq \fnorm{\T{B}\times_{n} (I - Q_{n} Q_{n}^\top)}^2 \\
		& = \fnorm{(I - Q_{n}Q_{n}^\top)\T{B}^{(n)}}^2 \\
		& = \fnorm{(U_n U_n^\top - Q_{n}Q_{n}^\top)\T{B}^{(n)} + (U_n)_\perp (U_n)_\perp^\top \T{B}^{(n)}}^2 \\
		& =  \underbrace{\fnorm{(U_n U_n^\top - Q_n Q_n^\top) \T{B}^{(n)}}^2}_{\textcircled{4}} + \underbrace{\fnorm{(U_n)_\perp (U_n)_\perp^\top \T{B}^{(n)}}^2}_{\textcircled{5}} + \underbrace{2 \mathrm{tr} \big((\T{B}^{(n)})^\top Q_n Q_n^\top (U_n)_\perp (U_n)_\perp^\top \T{B}^{(n)}\big)}_{\textcircled{6}}, \\
	\end{aligned}
	\nonumber
\end{equation}
in which $\textcircled{5}$ and $\textcircled{6}$ vanish when $r_n^\mathrm{true} \leq r_n$, since $(U_n)_\perp = 0$.

In the general case:
\begin{itemize}
	\item The error between projections of $\T{B}^{(n)}$ onto $U_n$ and $Q_n$ is
	\begin{equation}
		\begin{aligned}
			\textcircled{4}  & \leq \sigma_1 (\T{B}^{(n)})^2 \fnorm{U_n U_n^\top - Q_n Q_n^\top}^2 \\
			& \leq  4 \sigma_1 (B^{(n)})^2 r_n \cdot \frac{[2 \sigma_1 (\T{B}^{(n)}) + \norm{\bar{\T{X}}^{(n)}(\T{\Ph}) - \T{B}^{(n)}}]^2 \cdot \norm{\bar{\T{X}}^{(n)}(\T{\Ph}) - \T{B}^{(n)}}^2}{[\sigma_{r_n}^2 (\T{B}^{(n)}) - \sigma_{r_n+1}^2 (\T{B}^{(n)})]^2}, \\
		\end{aligned}
		\nonumber
	\end{equation}
	in which the last inequality comes from Lemma~\ref{lem:davis-kahan-variant}.
	
	\item The residual
	$\textcircled{5} = \sum_{i=r_n + 1}^{I_n} \sigma_i^2 (\T{B}^{(n)}) = (\tau_{r_n}^{(n)})^2$ is the $r_n$-th tail energy for $\T{B}^{(n)}$.	
	
	\item The inner product of projections is
	\begin{equation}
		\begin{aligned}
			\textcircled{6} & \leq 2 \norm{(\T{B}^{(n)})^\top \T{B}^{(n)}}_2 \cdot \mathrm{tr} \Big[[Q_n^\top (U_n)_\perp]^\top Q_n^\top (U_n)_\perp \Big] \\
			& \leq 2 \sigma_1 (\T{B}^{(n)})^2 \cdot \fnorm{Q_n^\top (U_n)_\perp}^2 \\
			& \leq 2 \sigma_1 (\T{B}^{(n)})^2 \cdot \Big\{\frac{2[2 \sigma_1 (\T{B}^{(n)}) + \norm{\bar{\T{X}}^{(n)}(\T{\Ph}) - \T{B}^{(n)}}] \min(r_n^{1/2} \norm{\bar{\T{X}}^{(n)}(\T{\Ph}) - \T{B}^{(n)}}, \fnorm{\bar{\T{X}}^{(n)}(\T{\Ph}) - \T{B}^{(n)}})}{\sigma_{r_n}^2 (\T{B}^{(n)}) - \sigma_{r_n+1}^2 (\T{B}^{(n)})} \Big\}^2\\
			& \leq 8 \sigma_1 (\T{B}^{(n)})^2 r_n \cdot \frac{[2 \sigma_1 (\T{B}^{(n)}) + \norm{\bar{\T{X}}^{(n)}(\T{\Ph}) - \T{B}^{(n)}}]^2 \cdot \norm{\bar{\T{X}}^{(n)}(\T{\Ph}) - \T{B}^{(n)}}^2}{[\sigma_{r_n}^2 (\T{B}^{(n)}) - \sigma_{r_n+1}^2 (\T{B}^{(n)})]^2},\\
		\end{aligned}
		\nonumber
	\end{equation}
	in which the first inequality comes from $\mathrm{tr}(AB) \leq \lambda_1(A) \mathrm{tr}(B)$ for positive semidefinite matrices $A$, $B$, and the second from last inequality comes from Lemma~\ref{lem:davis-kahan-variant}.
\end{itemize}

Together, the above conclude the proof for Theorem~\ref{thm:main-theorem-formal}.

\subsection{Proof for Theorem~\ref{thm:main-theorem-informal}, the special case}

Recall the high-probability upper bound of Theorem~\ref{thm:main-theorem-formal}, Equation~\ref{eq:main-thm-bound} is
\begin{equation}
	\begin{aligned}
		\frac{\fnorm{\widehat{\T{X}}(\widehat{\T{P}}) - \T{B}}^2}{\fnorm{\T{B}}^2} \leq & \min_{n \in [N]} \Bigg\{r_n \cdot \Big[\frac{\fnorm{\bar{\T{X}}(\T{\Ph}) - \bar{\T{X}}}}{\fnorm{\T{B}}} + \epsilon \Big]^2 \Bigg\} \\
		& + \sum_{n=1}^N \frac{12 r_n \sigma_1 (\T{B}^{(n)})^2}{\fnorm{\T{B}}^2}  \cdot \Bigg\{\frac{[2 \sigma_1 (\T{B}^{(n)}) + \fnorm{\bar{\T{X}}(\T{\Ph}) - \bar{\T{X}}} + \epsilon \fnorm{\T{B}}]^2}{[\sigma_{r_n} (\T{B}^{(n)}) + \sigma_{r_n+1} (\T{B}^{(n)})]^2} \cdot \frac{[\fnorm{\bar{\T{X}}(\T{\Ph}) - \bar{\T{X}}} + \epsilon \fnorm{\T{B}}]^2}{[\sigma_{r_n} (\T{B}^{(n)}) - \sigma_{r_n+1} (\T{B}^{(n)})]^2} \Bigg\}\\
		& + \frac{1}{\fnorm{\T{B}}^2} \sum_{n=1}^{N} (\tau_{r_n}^{(n)})^2. \\
	\end{aligned}
	\nonumber
\end{equation}

We denote $f(n) \sim g(n)$ if there exist universal constants $C_1, C_2$ and $N_0$ such that $C_1 g(n) \leq f(n) \leq C_2 g(n)$ for each $n > N_0$.

For an order-$N$ cubical tensor $\T{B}$ with size $I_1 = \cdots = I_N = I$, multilinear rank $r_1^\mathrm{true} = \cdots = r_N^\mathrm{true} = r < I$, and target multilinear rank $(r, r, \ldots, r)$, we choose $\epsilon \sim \sqrt{\frac{N \log I}{I}}$.
In this scenario:
\begin{itemize}
	\item From Lemma~\ref{lem:propensity-error-on-x-bar}, we have
	\begin{equation}
		\frac{\fnorm{\bar{\T{X}}(\T{\Ph}) - \bar{\T{X}}}}{\fnorm{\T{B}}} \leq \frac{\alpha_\mathrm{sp}}{\sigma(-\gamma) \sigma(-\alpha)} \sqrt{4 e L_\gamma \tau \Big(\frac{1}{\sqrt{I_\square}}+\frac{1}{\sqrt{I_{\square^C}}}\Big)} \sim I^{-N/8} = O(\epsilon).
		\nonumber
	\end{equation}
	\item When $I \geq r N \log I$, $\epsilon \fnorm{\T{B}^{(n)}} = O(\frac{1}{\sqrt{r}} \fnorm{\T{B}^{(n)}}) = O(\sigma_1 (\T{B}^{(n)}))$ for every $n \in [N]$. 
	\item For every $n \in [N]$, the tail singular values $\sigma_j (\T{B}^{(n)}) = 0$ for $j = r+1, \ldots, I$.
\end{itemize}

Thus in the upper bound of Theorem~\ref{thm:main-theorem-formal}, Equation~\ref{eq:main-thm-bound} above:
\begin{itemize}
	\item The first term
	\begin{equation}
		\begin{aligned}
			\min_{n \in [N]} \Bigg\{r_n \cdot \Big[\frac{\fnorm{\bar{\T{X}}(\T{\Ph}) - \bar{\T{X}}}}{\fnorm{\T{B}}} + \epsilon \Big]^2 \Bigg\} = O(4 r \epsilon^2).
		\end{aligned}
		\nonumber
	\end{equation}
	\item In the proof of Theorem~\ref{thm:main-theorem-formal}, Term $\textcircled{5}$ and $\textcircled{6}$ vanish when $r_n^\mathrm{true} \leq r_n$, since $(U_n)_\perp = 0$. 
	Together with $\frac{\sigma_{1} (\T{B}^{(n)})}{\sigma_r (\T{B}^{(n)})} \leq \kappa$ for every $n \in [N]$, the second term in the upper bound of Equation~\ref{eq:main-thm-bound}
	\begin{equation}
		\begin{aligned}
			& \sum_{n=1}^N \frac{4 r_n \sigma_1 (\T{B}^{(n)})^2}{\fnorm{\T{B}}^2}  \cdot \Bigg\{\frac{[2 \sigma_1 (\T{B}^{(n)}) + \fnorm{\bar{\T{X}}(\T{\Ph}) - \bar{\T{X}}} + \epsilon \fnorm{\T{B}}]^2}{[\sigma_{r_n} (\T{B}^{(n)}) + \sigma_{r_n+1} (\T{B}^{(n)})]^2} \cdot \frac{[\fnorm{\bar{\T{X}}(\T{\Ph}) - \bar{\T{X}}} + \epsilon \fnorm{\T{B}}]^2}{[\sigma_{r_n} (\T{B}^{(n)}) - \sigma_{r_n+1} (\T{B}^{(n)})]^2} \Bigg\} \\
			& \leq \sum_{n=1}^N \frac{4 r \sigma_1 (\T{B}^{(n)})^2}{\fnorm{\T{B}}^2} \cdot \Bigg\{\frac{[4 \sigma_1 (\T{B}^{(n)})]^2}{\sigma_{r_n}^2 (\T{B}^{(n)})} \cdot \frac{(2 \epsilon \fnorm{\T{B}})^2}{\sigma_{r_n}^2 (\T{B}^{(n)})} \Bigg\} \\
			& \leq 256 N r \kappa^4 \epsilon^2.
		\end{aligned}
		\nonumber
	\end{equation}
	\item The third term $ \frac{1}{\fnorm{\T{B}}^2} \sum_{n=1}^{N} (\tau_{r_n}^{(n)})^2 = 0$.
\end{itemize}

Together, we have the simplified high-probability upper bound
\begin{equation}
	\begin{aligned}
		& \frac{\fnorm{\widehat{\T{X}}(\widehat{\T{P}}) - \T{B}}}{\fnorm{\T{B}}} \leq \epsilon \sqrt{4 r + 256 N r \kappa^4 } = O\Big( N \sqrt{\frac{r \log I}{I}} \Big).
	\end{aligned}
	\label{eq:cor-of-main-thm-bound}
	\nonumber
\end{equation}

As for the probability lower bound $\displaystyle 1 - \frac{C_1}{I_\square + I_{\square^C}} - \sum_{n=1}^N [I_n + I_{(-n)}] \exp \Big[- \frac{\epsilon^2 \fnorm{\T{B}}^2 \sigma(- \alpha)/ 2}{I_{(-n)} \psi^2 + \epsilon \psi \fnorm{\T{B}}/3}\Big]$:
\begin{itemize}
	\item With the universal constant $C_1 > 0$, we have $\frac{C_1}{I_\square + I_\square^C} = O(I^{-1})$.
	\item The sum of probabilities from the matrix Bernstein inequality
	\begin{equation}
		\begin{aligned}
			\sum_{n=1}^N [I_n + I_{(-n)}] \exp \Big[- \frac{\epsilon^2 \fnorm{\T{B}}^2 \sigma(- \alpha)/ 2}{I_{(-n)} \psi^2 + \epsilon \psi \fnorm{\T{B}}/3}\Big] 
			& = O(N I^{N-1} \cdot \exp \Big[- \frac{\epsilon^2 \fnorm{\T{B}}^2}{I^{N-1}}\Big]) \\
			& = O(N I^{N-1} \cdot \exp(-2 \epsilon^2 I)) \\
			& = O(N I^{N-1} \cdot I^{-2 N}) \\
			& = O(I^{-1}).
		\end{aligned}
		\nonumber
	\end{equation}
\end{itemize}
Thus the probability is at least $1 - I^{-1}$.
This concludes the proof for Theorem~\ref{thm:main-theorem-informal}.

\section{Gradient computation for \textsc{NonconvexPE} (Algorithm~\ref{alg:propensity_estimation_alt})}
\label{sec:gradients}
For any $y \in \RR$ and $X \in \RR^{m \times n}$, we define the scalar-to-matrix derivative $\partial y / \partial X$ as a matrix of the same size as $X$, with the $(i, j)$-th entry $[\partial y / \partial X]_{ij} = \partial y / \partial X_{ij}$ for every $i \in [m]$, $j \in [n]$.

Recall that in \textsc{NonconvexPE}, we are using the gradient descent algorithm to minimize
\begin{equation}
	\begin{aligned}
		f (\T{G}^{\T{A}}, \{U_n^{\T{A}}\}_{n \in [N]}) = \sum_{i_1}^{I_1}  \cdots \sum_{i_N}^{I_N} & -\Upomega_{i_1 \cdots i_N} \log \sigma[(\T{G}^{\T{A}} \times_1 U_1^{\T{A}} \times_2 \cdots \times_N U_N^{\T{A}})_{i_1 \cdots i_N}] \\
		& -(1-\Upomega_{i_1 \cdots i_N})\log\{1-\sigma[(\T{G}^{\T{A}} \times_1 U_1^{\T{A}} \times_2 \cdots \times_N U_N^{\T{A}})_{i_1 \cdots i_N}]\},
	\end{aligned}
	\label{eq:objective-function-for-gd}
\end{equation}
in which $\sigma$ is the link function.
Denote $\widehat{\T{A}} := \T{G}^{\T{A}} \times_1 U_1^{\T{A}} \times_2 \cdots \times_N U_N^{\T{A}}$.
When we use the logistic link function $\sigma(x) = 1/(1 + e^{-x})$, $f$ is the sum of entry-wise logistic losses between the true binary mask tensor $\Upomega$ and the observation probability tensor $\sigma(\widehat{\T{A}})$.

We first show the gradient of the logistic loss, and we omit the calculations. 
\begin{lemma}(Gradient of the logistic loss)
	\label{lem:gradient-of-log-loss}
	For the logistic loss $\ell(x, y) = -y \log \sigma(x) - (1 - y) \log (1 - \sigma(x))$, we have $\partial \ell / \partial x = \sigma(x) - y$.
\end{lemma}

We then show Lemma~\ref{lem:gradient-of-scalar-to-matrix} for the chain rule of gradients of real-valued functions over matrices.
\begin{lemma}(Chain rule of scalar-to-matrix derivatives)
	\label{lem:gradient-of-scalar-to-matrix}
	Let $A$ be a matrix of size $m \times n$, and $g: \RR \rightarrow \RR$ be a continuously differentiable function.
	Define the real-valued function $\tilde{G}: \RR^{m \times n} \rightarrow \RR$ as 
	\begin{equation}
		\tilde{G}(A) = \sum_{i=1}^m \sum_{j=1}^n g(A_{ij}).
		\nonumber
	\end{equation}
	Then:
	\begin{enumerate}
		\item If $X, Y$ are matrices of size $m \times p$ and $p \times n$, respectively, and $A = X Y$, then 
		\begin{equation}
			\frac{\partial \tilde{G}(A)}{\partial X} = \frac{\partial \tilde{G}(A)}{\partial A} Y^\top.
			\nonumber
		\end{equation}
		\item If $X, Y, Z$ are matrices of size $m \times p$, $p \times q$ and $q \times n$, respectively, and $A = X Y Z$, then 
		\begin{equation}
			\frac{\partial \tilde{G}(A)}{\partial Y} = X^\top \frac{\partial \tilde{G}(A)}{\partial A} Z^\top.
			\nonumber
		\end{equation} 
	\end{enumerate}
\end{lemma}
\begin{proof}
	We show our proof in a similar fashion as \cite[Lemma 2]{hong2020generalized}. 
	In Case 1, 
	\begin{equation}
		\frac{\partial A_{kl}}{\partial X_{ij}} = \left\{
		\begin{aligned}
			& Y_{jl}, \quad&& \mathrm{if}\; k = i \\
			& 0, \quad && \mathrm{if }\; k \neq i \\
		\end{aligned}
		\right.
		\nonumber
	\end{equation}
	for every $k, i \in [m]$, $l \in [n]$, $j \in [p]$. 
	Thus
	\begin{equation}
		\begin{aligned}
			\frac{\partial \tilde{G}(A)}{\partial X_{ij}} & = \sum_{k=1}^m \sum_{l=1}^n \frac{\partial \tilde{G}(A)}{\partial A_{kl}} \frac{\partial A_{kl}}{\partial X_{ij}} \\
			& = \sum_{l=1}^n \frac{\partial \tilde{G}(A)}{\partial A_{il}} Y_{jl} = \Big(\frac{\partial \tilde{G}(A)}{\partial A} Y^\top \Big)_{ij}.
		\end{aligned}
		\nonumber
	\end{equation}
	In Case 2, since $\displaystyle A_{kl} = \sum_{i=1}^m \sum_{j=1}^n X_{ki} Y_{ij} Z_{jl}$, we have $\displaystyle \frac{\partial A_{kl}}{\partial Y_{ij}} = X_{ki} Z_{jl}$. Thus
	\begin{equation}
		\begin{aligned}
			\frac{\partial \tilde{G}(A)}{\partial Y_{ij}} & =\sum_{k=1}^m \sum_{l=1}^n \frac{\partial \tilde{G}(A)}{\partial A_{kl}} \frac{\partial A_{kl}}{\partial Y_{ij}} \\
			& = \sum_{k=1}^m \sum_{l=1}^n X_{ki} \frac{\partial \tilde{G}(A)}{\partial A_{kl}} Z_{jl} \\
			& =\sum_{k=1}^m \sum_{l=1}^n (X^\top)_{ik} \frac{\partial \tilde{G}(A)}{\partial A_{kl}} (Z^\top)_{lj} \\
			& = \Big(X^\top \frac{\partial \tilde{G}(A)}{\partial A} Z^\top \Big)_{ij}.
		\end{aligned}
		\nonumber
	\end{equation}
	These conclude the proof for Lemma~\ref{lem:gradient-of-scalar-to-matrix} based on the definition of scalar-to-matrix derivatives.
\end{proof}

Finally, we show the gradients $\{\partial f / \partial U_n\}_{n \in [N]}$ and $\partial f / \partial \T{G}$ in Theorem~\ref{thm:gradient-of-f}.
\begin{theorem}(Gradients of the objective function in \textsc{NonconvexPE})
	\label{thm:gradient-of-f}
	For each $n \in [N]$, with 
	\begin{equation}
		\begin{aligned}
			f (\T{G}^{\T{A}}, \{U_n^{\T{A}}\}_{n \in [N]}) = \sum_{i_1}^{I_1} \sum_{i_2}^{I_2} \cdots \sum_{i_N}^{I_N} & -\Upomega_{i_1 \cdots i_N} \log \sigma[(\T{G}^{\T{A}} \times_1 U_1^{\T{A}} \times_2 \cdots \times_N U_N^{\T{A}})_{i_1 \cdots i_N}] \\
			& -(1-\Upomega_{i_1 \cdots i_N})\log\{1-\sigma[(\T{G}^{\T{A}} \times_1 U_1^{\T{A}} \times_2 \cdots \times_N U_N^{\T{A}})_{i_1 \cdots i_N}]\},
		\end{aligned}
		\nonumber
	\end{equation}
	and $\widehat{\T{A}} = \T{G}^{\T{A}} \times_1 U_1^{\T{A}} \times_2 \cdots \times_N U_N^{\T{A}}$, we have:
	\begin{enumerate}
		\item The gradient with respect to the factor matrix $U_n$
		\begin{equation}
			\begin{aligned}
				\frac{\partial f}{\partial U_n^{\T{A}}} = \frac{\partial f}{\partial \widehat{\T{A}}^{(n)}} \cdot \big(U_{n+1}^{\T{A}} \otimes U_{n+2}^{\T{A}} \otimes \cdots \otimes U_{N}^{\T{A}} \otimes U_1^{\T{A}} \otimes U_2^{\T{A}} \otimes \cdots \otimes U_{n-1}^{\T{A}} \big) \cdot [(\T{G}^{\T{A}})^{(n)}]^\top.
				\nonumber
			\end{aligned}
		\end{equation}
		\item The gradient with respect to the unfolded core tensor $(\T{G}^{\T{A}})^{(n)}$
		\begin{equation}
			\begin{aligned}
				\frac{\partial f}{\partial (\T{G}^{\T{A}})^{(n)}} = (U_n^{\T{A}})^\top \cdot \frac{\partial f}{\partial \widehat{\T{A}}^{(n)}} \cdot \big(U_{n+1}^{\T{A}} \otimes U_{n+2}^{\T{A}} \otimes \cdots \otimes U_{N}^{\T{A}} \otimes U_1^{\T{A}} \otimes U_2^{\T{A}} \otimes \cdots \otimes U_{n-1}^{\T{A}} \big).
			\end{aligned}
			\nonumber
		\end{equation}
	\end{enumerate}
\end{theorem}
\begin{proof}
	With the Tucker decomposition of $\widehat{\T{A}}$, we have $\widehat{\T{A}}^{(n)} = U_n^{\T{A}} \cdot (\T{G}^{\T{A}})^{(n)} \cdot \big(U_{n+1}^{\T{A}} \otimes U_{n+2}^{\T{A}} \otimes \cdots \otimes U_{N}^{\T{A}} \otimes U_1^{\T{A}} \otimes U_2^{\T{A}} \otimes \cdots \otimes U_{n-1}^{\T{A}} \big)^\top$ for the unfolding in each of the $n \in [N]$ \cite{de2000multilinear}.
	Thus we can apply each case of Lemma~\ref{lem:gradient-of-scalar-to-matrix} to the corresponding case here, with $A$ to be $\widehat{\T{A}}^{(n)}$.  
\end{proof} 

With Lemma~\ref{lem:gradient-of-log-loss}, we have $\partial f / \partial \widehat{\T{A}} = \sigma(\widehat{\T{A}}) - \Upomega$ for the logistic link function $\sigma$. 
This can be inserted into Theorem~\ref{thm:gradient-of-f} for the gradients $\{\partial f / \partial U_n\}_{n \in [N]}$ and $\partial f / \partial \T{G}$, but note that Theorem~\ref{thm:gradient-of-f} does not rely on this result.

\section{Sensitivity of propensity estimation algorithms to hyperparameters}
\label{sec:algorithm_sensitivity}
We study the sensitivities of \textsc{ConvexPE} (Algorithm~\ref{alg:propensity_estimation_provable}) and \textsc{NonconvexPE} (Algorithm~\ref{alg:propensity_estimation_alt}) to their respective hyperparameters. 

The most important hyperparameters in \textsc{ConvexPE} are $\tau$ and $\gamma$.
Ideally, we want to set $\tau=\theta$ and $\gamma=\alpha$; this is not possible in practice, though, since we do not know the $\theta$ and $\alpha$ of the true parameter tensor $\T{A}$. 
In the setting of the third experiment in Section~\ref{sec:experiments-synthetic} of the main paper, we study the relationship between relative errors of propensity estimates and the ratios $\tau / \theta$ and $\gamma / \alpha$ in Figure~\ref{fig:prox-prox-hyperparameter-sensitivity}.
We can see that the performance is much more sensitive to $\tau$ than $\gamma$, and a slight deviation of $\tau / \theta$ from 1 results in a much larger propensity estimation error. 

\begin{figure}
	\centering
\begin{subfigure}[t]{.28\linewidth}
	\includegraphics[width=\linewidth]{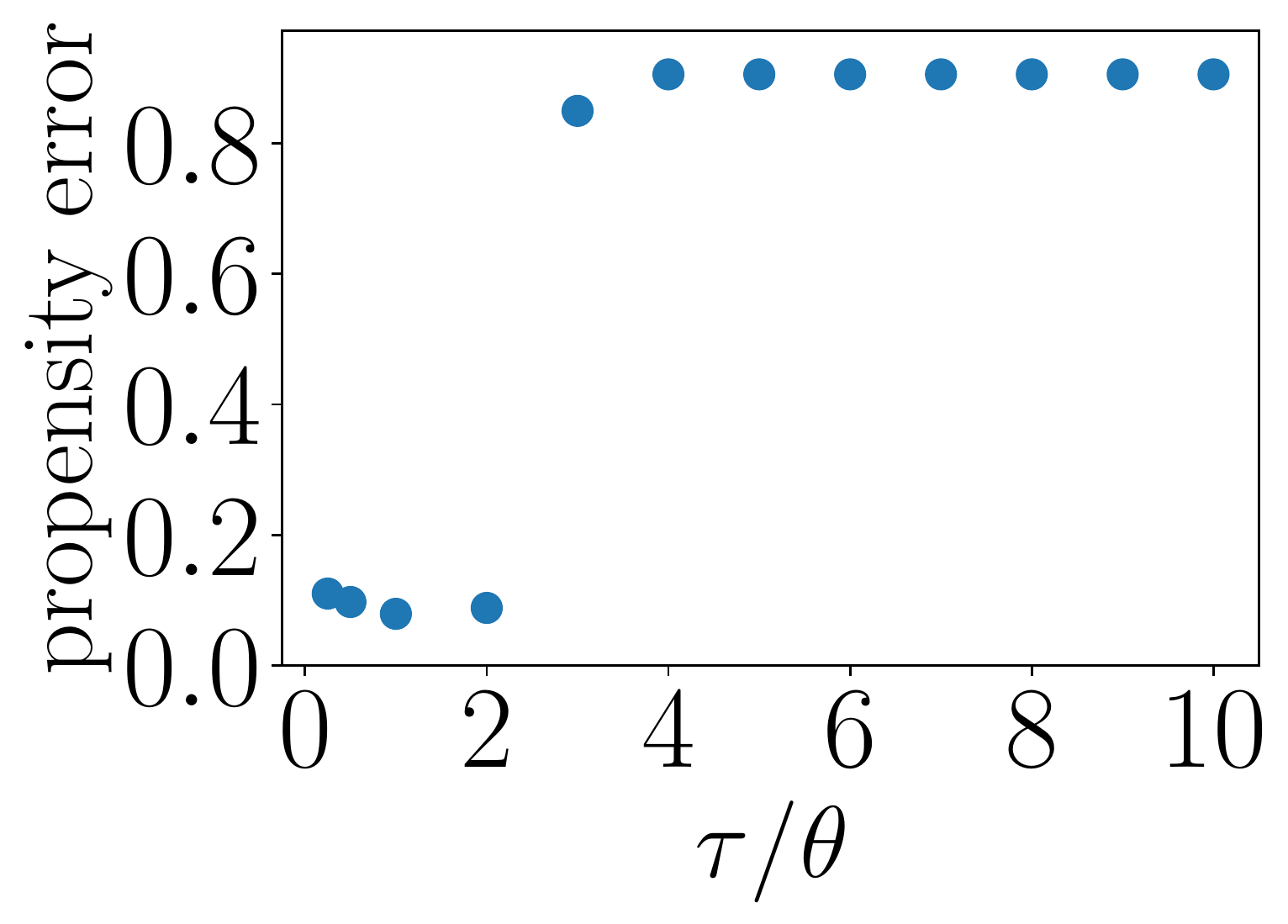}
	\caption{sensitivity to $\tau$ when $\gamma / \alpha=1$}
\label{fig:prox-prox-hyperparameter-sensitivity-tau}	
\end{subfigure}
\hspace{.03\linewidth}
\begin{subfigure}[t]{.28\linewidth}
	\includegraphics[width=\linewidth]{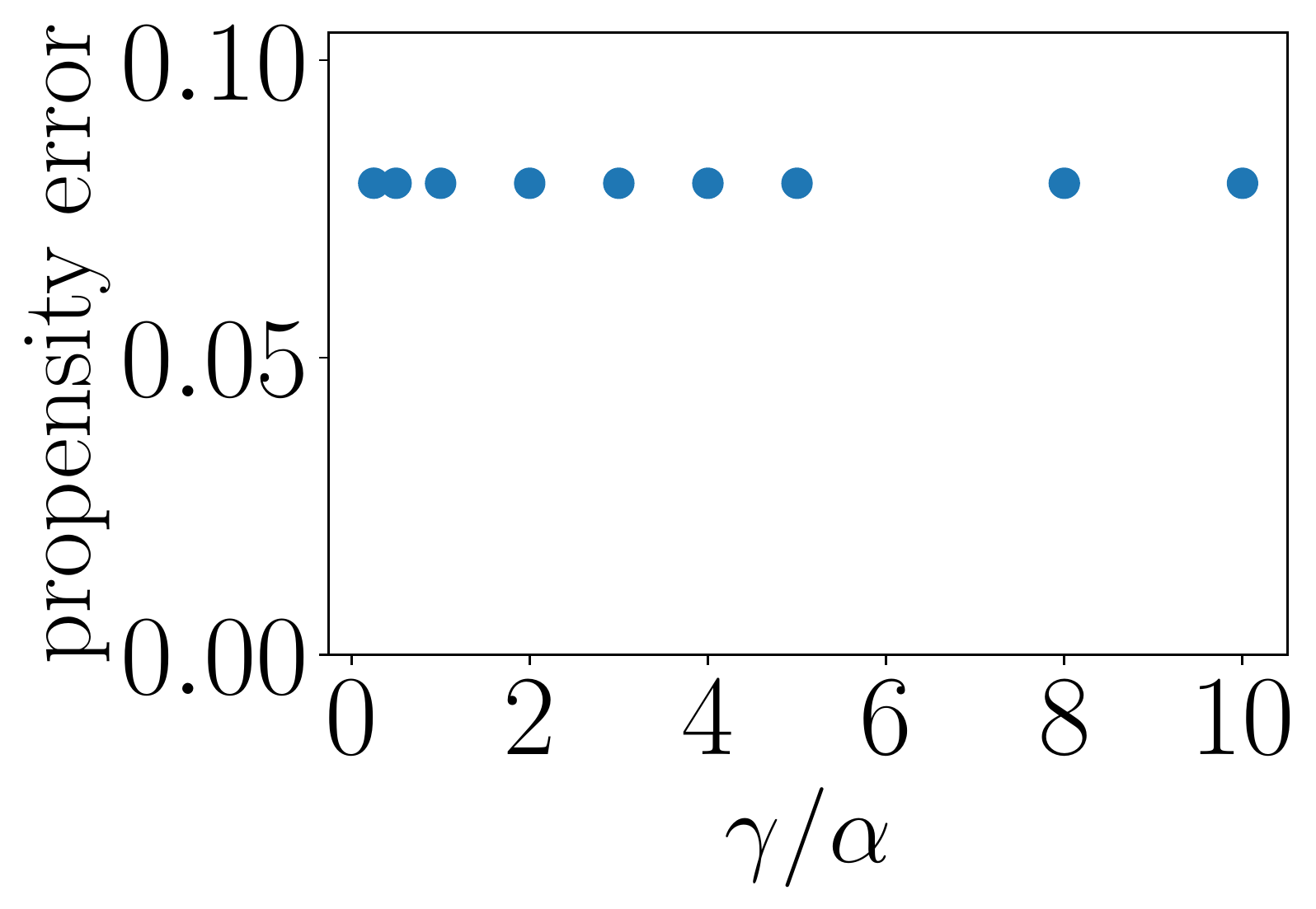}
	\caption{sensitivity to $\gamma$ when $\tau / \theta=1$}
\label{fig:prox-prox-hyperparameter-sensitivity-gamma}	
\end{subfigure}
	\caption{Sensitivity of \textsc{ConvexPE} (Algorithm~\ref{alg:propensity_estimation_provable}) to hyperparameters $\tau$ and $\gamma$.}
	\label{fig:prox-prox-hyperparameter-sensitivity}
\end{figure}

The most important hyperparameter in \textsc{NonconvexPE} is the step size $t$.
We show both the convergence and the change of propensity relative errors of \textsc{NonconvexPE} at several step sizes in Figure~\ref{fig:gd-hyperparameter-sensitivity}.
We can see that the relative errors of propensity estimates steadily decrease at all step sizes at which the gradient descent converges.
Also, the respective rankings of relative losses and propensity errors at different step sizes are the same across all iterations, indicating that the relative loss is a good surrogate metric for us to seek a good propensity estimate. 
Thus practitioners can select the largest step size at which \textsc{NonconvexPE} converges; it is $5 \times 10^{-6}$ in our practice.
This is much easier than the selection of $\tau$ in \textsc{ConvexPE}. 

\begin{figure}
\centering
\begin{minipage}[b]{0.75\linewidth}		
\begin{subfigure}[t]{.45\linewidth}
\includegraphics[width=\linewidth]{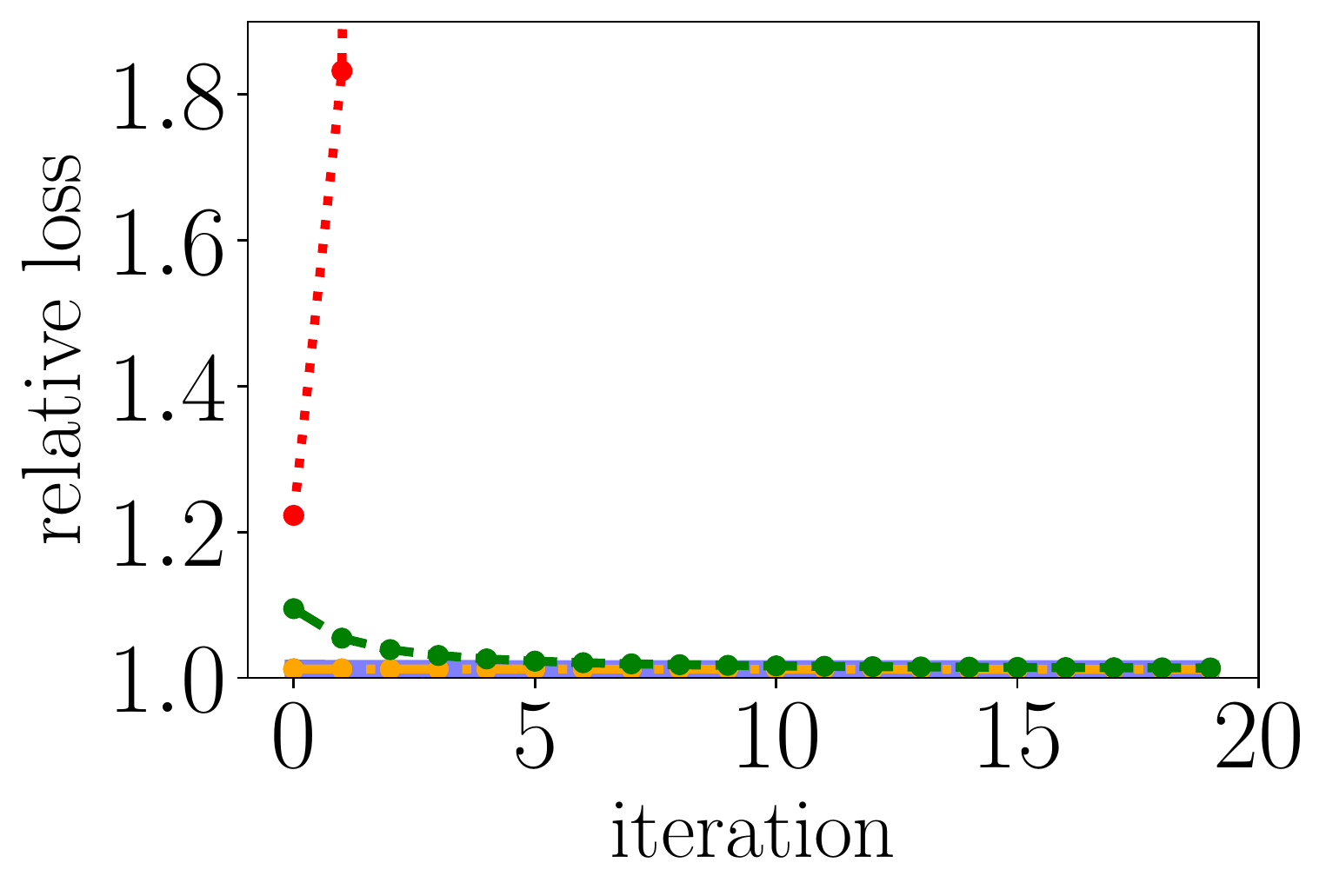}
\caption{change of relative loss}
\label{fig:gd-hyperparameter-sensitivity-loss}
\end{subfigure}
	\hspace{.03\linewidth}	
\begin{subfigure}[t]{.45\linewidth}
\includegraphics[width=\linewidth]{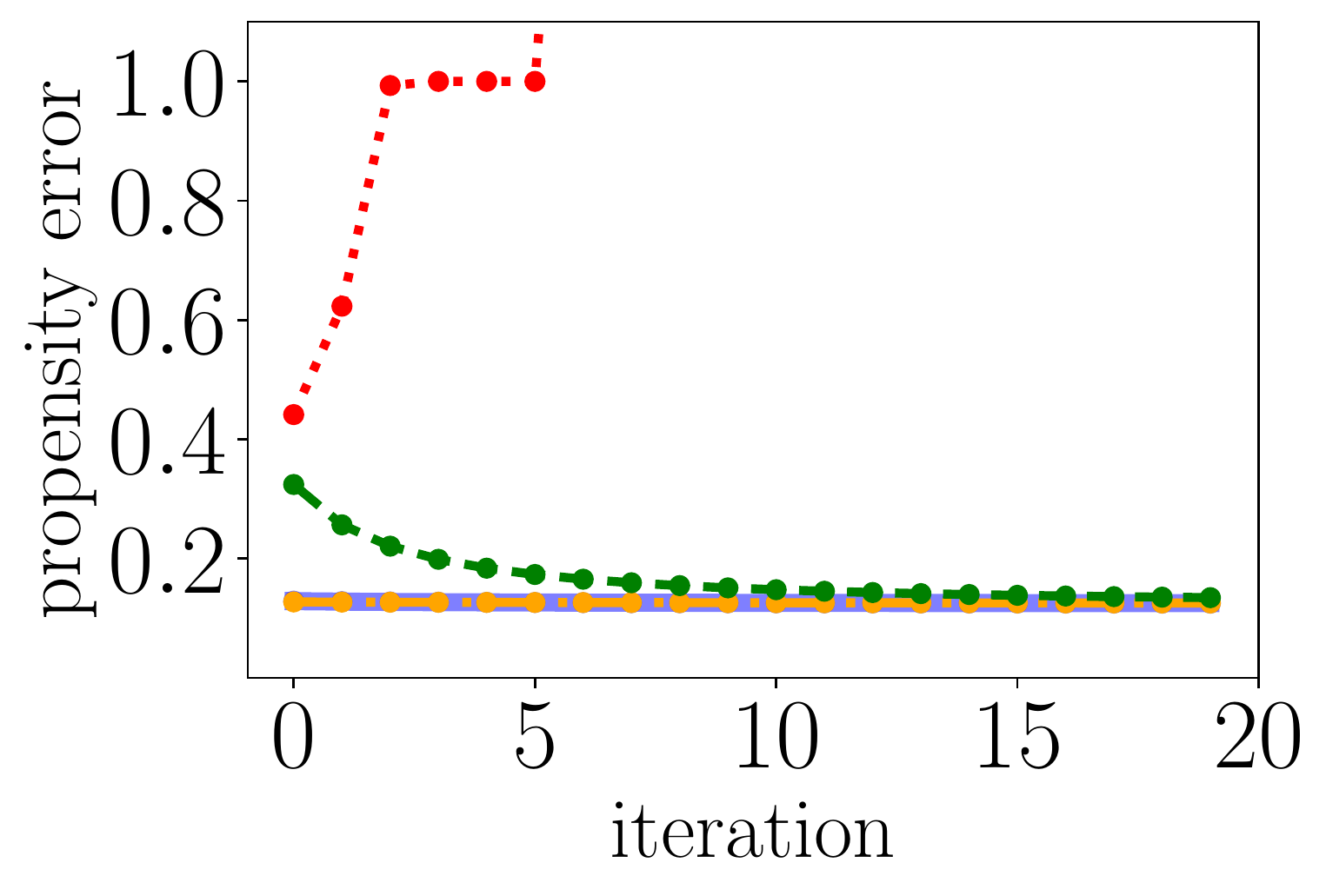}
\caption{change of propensity estimation error}
\label{fig:gd-hyperparameter-sensitivity-error}
\end{subfigure}
\end{minipage}
\begin{minipage}[b]{.16\linewidth}
	\stackunder[1pt]{\includegraphics[width=\linewidth]{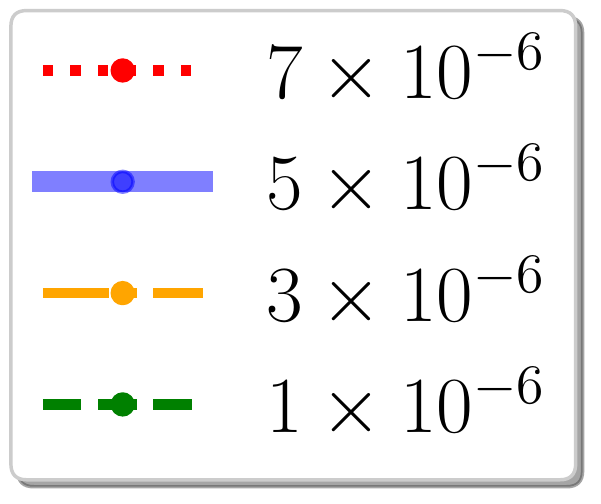}}{}
\end{minipage}
	
	\caption{Sensitivity of \textsc{NonconvexPE} (Algorithm~\ref{alg:propensity_estimation_alt}) to step size $t$.
		Since the objective function is the logistic loss between the mask tensor $\Upomega$ and the parameter tensor $\T{A}$, the relative loss in Figure~\ref{fig:gd-hyperparameter-sensitivity-loss} is the ratio of actual logistic loss to the best logistic loss computed from the true parameter tensor.
		Propensity error in Figure~\ref{fig:gd-hyperparameter-sensitivity-error} is $\fnorm{\widehat{\T{P}} - \T{P}} / \fnorm{\T{P}}$, the same as in the main paper.}
	\label{fig:gd-hyperparameter-sensitivity}
\end{figure}

\end{document}